\documentclass[twoside]{article}

\usepackage[accepted]{aistats2024}

\usepackage[round]{natbib}

\usepackage{color,xcolor}
\definecolor{mydarkblue}{rgb}{0,0.08,0.45}
\usepackage[colorlinks=true,
    linkcolor=mydarkblue,
    citecolor=mydarkblue,
    filecolor=mydarkblue,
    urlcolor=mydarkblue]{hyperref}
\usepackage{url}            %
\usepackage{multirow}
\usepackage{booktabs}       %
\usepackage{amsmath,amsfonts,amssymb}
\usepackage{mathtools}
\usepackage{algorithm}
\usepackage{algorithmic}
\usepackage[capitalize,noabbrev]{cleveref}
\usepackage{tikzducks}
\usepackage{float}
\usepackage{enumitem}
\usepackage{subfig}
\usepackage{xcolor}
\usepackage{colortbl}
\usepackage{amsthm}
\usepackage{thmtools,thm-restate}
\usepackage{nicematrix}
\usepackage{bbm}

\newtheorem{theorem}{Theorem}

\newtheorem{assumption}{Assumption}
\newtheorem{lemma}{Lemma}

\theoremstyle{definition}
\newtheorem{example}{Example}
\AtBeginEnvironment{example}{%
  \pushQED{\qed}%
}
\AtEndEnvironment{example}{\popQED\endexample}
\newtheorem{definition}{Definition}

\DeclareMathOperator{\mb}{mb}
\DeclareMathOperator{\pa}{pa}
\DeclareMathOperator{\ch}{ch}
\DeclareMathOperator{\sps}{sps}

\DeclareMathOperator{\des}{des}
\DeclareMathOperator{\rank}{rank}

\newcommand{\Sb}{\mathbf{S}}
\newcommand{\I}{\mathbf{I}}
\newcommand{\A}{\mathbf{B}}
\newcommand{\B}{\mathbf{A}}
\newcommand{\C}{\mathbf{C}}
\newcommand{\D}{\mathbf{D}}
\newcommand{\W}{\mathbf{W}}
\newcommand{\X}{\mathbf{X}}
\newcommand{\E}{\mathbf{E}}
\newcommand{\G}{\mathcal{G}}
\newcommand{\M}{\mathbf{M}}
\newcommand{\U}{\mathbf{U}}
\newcommand{\K}{\mathbf{K}}
\newcommand{\Y}{\mathbf{Y}}
\newcommand{\Z}{\mathbf{Z}}
\newcommand{\Pb}{\mathbf{P}}

\newcommand{\Hb}{\mathbf{H}}

\newcommand{\PRT}{\mathbf{\Gamma}}
\newcommand\indep{\protect\mathpalette{\protect\independenT}{\perp}}
\def\independenT#1#2{\mathrel{\rlap{$#1#2$}\mkern2mu{#1#2}}}

\usepackage{wrapfig}
\usepackage{tikz}
\usetikzlibrary{arrows,positioning}

\makeatother

\begin{document}

\runningauthor{Haoyue Dai, Ignavier Ng, Yujia Zheng, Zhengqing Gao, Kun Zhang}

\twocolumn[

\aistatstitle{Local Causal Discovery with Linear non-Gaussian Cyclic Models}

\aistatsauthor{ Haoyue Dai$^{*1}$ \And Ignavier Ng$^{*1}$ \And Yujia Zheng$^{1}$ \And Zhengqing Gao$^{2}$ \And Kun Zhang$^{1,2}$}

\aistatsaddress{\hspace{-6em}$^{1}$Carnegie Mellon University\And \hspace{-6em}$^{2}$Mohamed bin Zayed University of Artificial Intelligence } ]

\begin{abstract}
\vspace{-0.6em}
Local causal discovery is of great practical significance, as there are often situations where the discovery of the global causal structure is unnecessary, and the interest lies solely on a single target variable. Most existing local methods utilize conditional independence relations, providing only a partially directed graph, and assume acyclicity for the ground-truth structure, even though real-world scenarios often involve cycles like feedback mechanisms. In this work, we present a general, unified local causal discovery method with linear non-Gaussian models, whether they are cyclic or acyclic. We extend the application of independent component analysis from the global context to independent subspace analysis, enabling the exact identification of the equivalent local directed structures and causal strengths from the Markov blanket of the target variable. We also propose an alternative regression-based method in the particular acyclic scenarios. Our identifiability results are empirically validated using both synthetic and real-world datasets.\looseness=-1
\end{abstract}
\vspace{-1.0em}
\section{INTRODUCTION}\label{sec:intro}
\vspace{-0.1em}

Causal discovery aims to identify causal relations among variables from data. In many real-world scenarios, it is not essential to determine the causal structure across all variables. Rather, the primary interest is often in unveiling the causes and effects related to specific target variables. Allocating resources to estimate a global structure for such narrowed objectives can be computationally excessive. This is exemplified in scRNA-seq data, where attempting global causal discovery to derive the gene regulatory network amongst approximately $20$k genes is not only computationally expensive but also often redundant \citep{levine2005gene}.
Local causal discovery, emphasizing the causal relations of a target variable and its neighbors, stands out as a more grounded and efficient approach. Additionally, techniques like divide-and-conquer and parallelization, when applied through local causal discovery, can often enhance the efficiency of identifying the global causal structure~\citep{ma2023local}.

\looseness=-1
Building on this motivation, several studies have delved into the discovery of local structures within a select subset of variables \citep{margaritis1999bayesian, yin2008partial, zhou2010discover, niinimaki2012local, wang2014discovering, gao2015local,gao2017local, ling2020using, ng2021reliable, yu2021feature, gupta2023local}. The distinction in this line of research lies in the estimation approaches
used to estimate these local structures, such as parent-child sets.
These approaches range from testing conditional independence relations to employing likelihood-based score functions. With appropriate tests or scores, they can offer nonparametric guarantees. Yet, without parametric assumptions, both independence tests and score functions cannot uniquely determine all directions, leading to some edges being undirected.

\looseness=-1
Moreover, most existing work in local causal discovery assume that there are no cycles in the ground-truth structure. This constrains its applicability given that cycles frequently appear in real-world contexts. These cycles can arise from various origins, including feedback mechanisms in biological systems \citep{benito2007transcriptional}, electrical engineering \citep{mason1953feedback}, or economic processes \citep{haavelmo1943statistical}. Such cyclic relationships can have profound implications, reshaping our understanding of the systems under consideration. Furthermore, in local context, one often cannot make the assumption of global acyclicity, since there is no way for the acyclicity beyond the considered subset of variables to be testable. While there has been a steady progress on causal discovery with cycles \citep{spirtes1995directed,richardson1996discovering, lacerda2008discovering, hyttinen2012learning, mooij2013cyclic, ghassami2020characterizing}, none have offered methodologies with theoretical guarantees in the context of local search.

\textbf{Contributions.} \ \
To our knowledge, this work is the first to tackle local causal discovery in cyclic models, crucial for gene regulatory networks with prevalent feedback loops and numerous genes. Moreover, we allow intersecting cycles, a known challenging case. By leveraging non-Gaussianity, our approach determines causal directions and strengths, standing in contrast to most previous local methods that only identify partially directed edges. Notably, this work offers a unified perspective on acyclic~\citep{shimizu2006lingam,shimizu2011directlingam} and cyclic~\citep{lacerda2008discovering} cases within the local context. We establish identifiability guarantees for all proposed methods, and our theoretical results have been validated in both synthetic and real-world data.\looseness=-1

\vspace{-0.1em}
\section{Problem Setup}\label{sec:problem_setting}
\subsection{Notations, Definitions, and the Goal}
Let $\G=(\mathcal{V},\mathcal{E})$ be a directed graph with the vertex set $\mathcal{V}=[d]\coloneqq\{1,2,\dots,d\}$ and the edge set $\mathcal{E}$. Denote a directed edge from vertex $i$ to vertex $j$ as $i\rightarrow j$.

Random variables $\X = (X_i)_{i=1}^d$ are generated by a linear non-Gaussian (LiNG) structural equation model (SEM)~\citep{lacerda2008discovering} w.r.t. the graph $\G$, described in the matrix form as
\begin{equation}\label{eq:linear_sem_adj}
\X = \A \X + \E,
\end{equation}
where $\E= (E_i)_{i=1}^d$ are mutually independent non-Gaussian \textit{exogenous noise} components, and $\A$ is the \textit{adjacency matrix}, with the entry $\A_{j,i}$ representing the \textit{direct causal effect} of $X_i$ on $X_j$. $\A_{j,i}\neq 0$ if and only if $i\rightarrow j\in \mathcal{E}$. Solving for $\X$ in~\cref{eq:linear_sem_adj} gives
\begin{equation}\label{eq:linear_sem_mixing}
\X = \B \E, \text{ with } \B\coloneqq (\I-\A)^{-1},
\end{equation}
i.e., $\X$ can also be expressed directly as a linear combination of the noises, through the \textit{mixing matrix} $\B$. Following~\citep{lacerda2008discovering}, we allow cycles in $\G$ under some mild assumptions (see \cref{sec:isa_ling}), and interpret $\X$ as the equilibrium of the dynamic system.

For a vertex $i\in \mathcal{V}$, denote its \textit{Markov blanket} (MB) in $\G$ as $\mb_\G(i)\coloneqq \pa_\G(i)\cup\ch_\G(i)\cup\sps_\G(i)$, the union of its \textit{parents} $\pa_\G(i)\coloneqq \{j\in\mathcal{V}: j\rightarrow i \in \mathcal{E}\}$, \textit{children} $\ch_\G(i)\coloneqq \{j\in\mathcal{V}: i\rightarrow j \in \mathcal{E}\}$, and \textit{spouses} $\sps_\G(i)\coloneqq \{j\in \mathcal{V}\backslash (\pa_\G(i)\cup\ch_\G(i)): \ch_\G(i) \cap \ch_\G(j) \neq \emptyset\}$. Assuming faithfulness, $\mb_\G(i)$ corresponds to the minimal set of variables conditioned on which all other variables are independent of $X_i$. Consequently, it is an appropriate starting point for the local search on vertex $i$.\looseness=-1

For a target vertex $T$, we provide a method in~\cref{sec:markov_blanket} to efficiently estimate $\mb_\G(T)$ from $\X$ even in the presence of cycles. Specifically, we generalize the method developed by \citet{loh2014high} in the acyclic case based on inverse covariance matrix of the distributions. Hence in the following main results (\cref{sec:inv_direct_lingam,sec:isa_ling}), we assume the oracle $\mb_\G(T)$ is available, and focus on the problem of further discovering causal effects related to $T$ from $T,\mb_\G(T),$ and their corresponding variables.\looseness=-1

\subsection{LiNG SEM and its Global Estimation}\label{subsec:global_cyclic_ling}

Our definition of a linear non-Gaussian (LiNG) cyclic model precisely follows~\citep{lacerda2008discovering}. We allow cycles in $\G$, interpret $\X$ as the equilibrium of the dynamic system. We allow overlapped cycles, but only assume that there are no ``self-loops'', i.e., $\A$ has all zeros in the diagonal, because by trivial scaling, any equilibrium even with self-loops can be equivalently entailed by another LiNG model without self-loops, as long as the self-loop strengths $\A_{i,i}\neq 1$. Moreover, we assume no cycles with strength exactly $1$, i.e., $\A$ has no eigenvalues of $1$, rendering $\I-\A$ invertible. See Section 1.2 of~\citep{lacerda2008discovering} for details.\looseness=-1

Recall that~\cref{eq:linear_sem_mixing}, $\X=\B\E$, is in the exact form of independent component analysis (ICA)~\citep{comon1994independent,hyvarinen2000independent}, where observed data $\X$ (signals) is an unknown linear invertible mixture of unknown non-Gaussian independent components $\E$ (blind sources). When all the variables in $\X$ are involved, namely, with \textit{causal sufficiency}, ICA can estimate a \textit{demixing matrix} $\W$ to separate $\X$ into independent components $\W\X$. It is shown that $\W$ identifies $\B^{-1}=\I-\A$ up to rows permutation and scaling. Interestingly, with the structural constraint of zero diagonals in $\A$ (i.e., diagonal ones in $\B^{-1}$), these indeterminacies can be further reduced. A row permutation is called \textit{admissible} if it makes $\W$ have diagonal ones with corresponding scaling. When $\G$ is acyclic, ICA-LiNGAM~\citep{shimizu2006lingam} shows that the admissible permutation is unique, resulting in the exact identification of $\A$. This is because for acyclic $\G$, its $\A$ can be simultaneously row and column permuted to be strictly lower triangular.\looseness=-1

\citet{lacerda2008discovering} generalizes ICA-LiNGAM to cyclic cases with almost a same algorithmic procedure: it begins with an ICA on $\X$ to obtain a demixing matrix $\W$, and then identifies the \textit{admissible} row permutations. The key distinction is that, in the presence of cycles in $\G$, there can be multiple admissible permutations. Denote the set of adjacency matrices recovered by all admissible permutations as $\mathcal{B}$, i.e.,\looseness=-1

\begin{restatable}{definition}{DEFLINGEQUIVCLASS}\label{def:LiNG_equiv_class}
For a LiNG model $\X=\A\X+\E$, denote\looseness=-1
    \begin{equation*}\label{eq:LiNG_equiv_class}
    \mathcal{B} \coloneqq \{\A': \A' = \I - \Pb^\pi\D\B^{-1}, \operatorname{diag}(\A')=\mathbf{0}\},
\end{equation*}
where $\B^{-1}=\I-\A$, $\D$ is an $d$-dim scaling matrix, and $\Pb^\pi$ is a permutation matrix (see~\cref{sec:isa_ling} for details) with $\pi$ enumerating permutations of $\mathcal{V}=[d]$.
\end{restatable}

Two different LiNG models from $\mathcal{B}$ entail a same equilibrium distribution and are termed \textit{distributionally equivalent}, though they share different graph structures; see Figure \ref{fig:cyclic_intro_examples} in Appendix \ref{sec:examples} for an example. Note that with linearity and non-Gaussianity, no two different acyclic SEMs are distributionally equivalent, guaranteeing the unique identification of $\A$ in LiNGAM, but there are different cyclic SEMs that are distributionally equivalent, and thus the true $\A$ can be identified up to an equivalence class. \citet{lacerda2008discovering} shows that $\mathcal{B}$, defined above from all admissible permutations, characterizes exactly the LiNG equivalence class for $\X$.\looseness=-1

\section{LOCAL LING DISCOVERY}\label{sec:isa_ling}
We develop a local causal discovery method based on independent subspace analysis (ISA), which enables the exact identification of the equivalent local directed structures and causal strengths from the MB of the target variable. We first explain how the commonly used ICA approach for discovering global causal structure \citep{shimizu2006lingam} fails in local context. We then describe the key identifiability result of ISA that we exploit, and provide a specific characterization of the ISA solution. Finally, we describe our proposed Local ISA-LiNG method, which involves (1) performing ISA on the local variables, (2) finding \textit{admissible} permutations on the ISA solutions, and (3) identifying local structures and coefficients from the permuted solutions. We prove that, interestingly, with only local variables, our proposed algorithm can identify exactly what can be identified globally with all variables.\looseness=-1

\subsection{Independent Subspace Analysis}\label{subsec:isa_definition}
Having introduced the cyclic LiNG and its ICA-based global estimation method, we now turn to our local case. When only a subset of variables (e.g., a target $T$ and its $\mb_\G(T)$) is involved, the main challenge lies in causal insufficiency: with hidden confounders, ICA cannot demix mutually independent components.\looseness=-1

\begin{example}\label{example:hidden_confounder_ica}
    In~\cref{fig:isa_ling_examples}(i), consider a target $T=4$ with $\mb_\G(T)=\{2,3\}$. With a confounder $X_1$ outside of $T$'s MB, ICA is not applicable on $\{X_2,X_3,X_4\}$, as these three signals mix four sources ($\{E_1,E_2,E_3,E_4\}$), and no invertible matrix $\W\in Gl(3)$ can separate out any three mutually independent components.\looseness=-1
\end{example}
Such an issue is typically pronounced in overcomplete ICA (OICA)~\citep{hyvarinen2000independent}, where the number of observed signals is less than the number of mixed sources. There are indeed work on LiNGAM with hidden confounders using OICA~\citep{hoyer2008estimation}, but OICA is known to be both computationally and statistically ineffective. In this work, our methods do not involve OICA, away from the difficulties of trying to separate out that many mutually independent sources from only a few signals. Instead, we only seek the separation ``as independent as possible'', and show that it is informative enough. To achieve this, independence subspace analysis (ISA)~\citep{hyvarinen2000emergence,theis2006general} comes into play.\looseness=-1

\begin{definition}\label{def:irreducibility}
    An $m$-dim random vector $\Z$ is called \textit{irreducible} if it contains no lower-dim independent components, i.e., no invertible matrix $\W\in Gl(m)$ can decompose $\W\Z = (\Z'_1,\Z'_2)$ into independent $\Z'_1 \indep \Z'_2$.\looseness=-1
\end{definition}

\begin{definition}\label{def:ISA_theis}
    For an $m$-dim random vector $\Y$, an invertible matrix $\W$ is called an \textit{independent subspace analysis (ISA)} solution of $\Y$ if $\W\Y = (\Z^\intercal_1, \dots, \Z^\intercal_k)^\intercal$ consists of mutually independent, irreducible random vectors $\Z_i$. The corresponding partition $\PRT_\W$ of indices $[m]$ is called the \textit{ISA partition} associated with $\W$.\looseness=-1
\end{definition}
Given a random vector $\Y$ with existing covariance and no Gaussian components, \citet{theis2006general} shows that an ISA solution of $\Y$ exists and, similar to ICA, is unique except for general scaling and permutation.

Before stating the result of ISA, we first introduce some notations. For a permutation $\pi:\{1,\dots,m\}\mapsto \{1,\dots,m\}$ of $m$ elements, let $\pi_i$ be the $i$-th element in $\pi$, and $\pi[j]$ be the index of element $j$ in $\pi$, i.e., $\pi_{\pi[j]}=j$. For an ordered subset $\Sb\subset [m]$, denote $\pi_\Sb \coloneqq (\pi_i: i\in\Sb)$ and $\pi[\Sb] \coloneqq (\pi[j]: j\in\Sb)$, where $(\cdot)$ means ordered sets. Define the $m\times m$ permutation matrix $\Pb^\pi$ by $\Pb^\pi_{i,j}=1$ if $\pi_i = j$ and $0$ otherwise. Given a partition $\PRT$ of $[m]$, an $m\times m$ block diagonal matrix $\D$ is said to be a \textit{general scaling matrix} consistent with $\PRT$, if $\forall \Sb \in \PRT$, $\rank(\D_{\Sb,\Sb})=|\Sb|$, and $\D_{\Sb,[m]\backslash\Sb}=\mathbf{0}$. Here the notation like $\D_{\Sb_1,\Sb_2}$ means the submatrix of $\D$ with rows and columns indexed by ordered sets $\Sb_1$ and $\Sb_2$ respectively. Subscripts on random vectors denotes indexing similarly, e.g., $\X_\Sb=(X_i:i\in\Sb)$. We have:\looseness=-1

\begin{theorem}[Indeterminacy of ISA; Theorem 1.8 of~\citep{theis2006general}]\label{thm:indeterminacies_of_isa}
    Given an $m$-dim random vector $\Y$, if both $\W_1$ and $\W_2$ are ISA solutions of $\Y$ with partitions $\PRT_{\W_1}, \PRT_{\W_2}$, then there exists a permutation $\pi$ of $[m]$ and a general scaling matrix $\D$ consistent with $\PRT_{\W_1}$ s.t. $\W_2 = \Pb^\pi \D \W_1$, and $\forall \Sb_1 \in \PRT_{\W_1}$, $\exists \Sb_2 \in \PRT_{\W_2}$, with $\Sb_2$ and $\pi[\Sb_1]$ having the same elements.\looseness=-1
\end{theorem}
ISA can be identified up to such indetermincaies, and can be estimated as efficiently as square ICA~\citep{theis2006general}. ICA can then be viewed as a special case of ISA, where all subspaces are of one-dimension.\looseness=-1

\subsection{One Specific ISA Characterization}\label{subsec:BMM_inv}
Given a vertex subset $\Sb \subset \mathcal{V}$ and the corresponding variables $\X_\Sb$, \cref{subsec:isa_definition} shows that although ICA on $\X_\Sb$ may not be applicable, an ISA solution  exists and is unique up to some indeterminacies. However, what is such an ISA solution? In the causally sufficient (i.e, ICA) case, a demixing matrix $\B^{-1} = \I-\A$ follows naturally from~\cref{eq:linear_sem_mixing}, while this is less obvious in the ISA case. Below we give a specific characterization.\looseness=-1

\begin{figure}[t]
\centering
\includegraphics[width=0.32\textwidth]{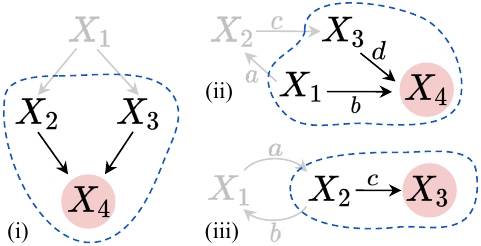}
\caption{For~\cref{example:hidden_confounder_ica,example:BSS_inv_vs_ASS,example:BSS_inv_no_diagonal_ones}. On each $\G$, the target $T$ is colored red, and its $\mb_\G(T)$ is circled by blue and colored dark. Same marks apply henceforth.\looseness=-1}
\label{fig:isa_ling_examples}
\end{figure}

\begin{restatable}[One characterization of ISA in LiNG model]{theorem}{THMONECHARACTERIZATIONOFISAINLING}\label{thm:BSS_inv_is_ISA}
    Assume $\X$ follows a LiNG SEM $\X=\B\E$. For any vertex subset $\Sb \subset \mathcal{V}$, the inverse of the principal submatrix of the mixing matrix $\B$ indexed by $\Sb$, denoted by $\B_{\Sb,\Sb}^{-1}$, is an ISA solution of $\X_\Sb$.
\end{restatable}

\cref{thm:BSS_inv_is_ISA} characterizes a specific ISA matrix $\B_{\Sb,\Sb}^{-1}$ that separates $\X_\Sb$ ``as independent as possible'', i.e., $\B_{\Sb,\Sb}^{-1} \X_\Sb$ produces irreducible independent subspaces. The proof is in~\cref{proof:thm_BSS_inv_is_ISA}. However, before delving into further identification of $\B_{\Sb,\Sb}^{-1}$, let us first closely examine and understand what it represents.\looseness=-1

Recall that in ICA, the adjacency matrix $\A$ that represents the causal structure and strengths can be directly read off of the demixing characterization, $\B^{-1}=\I-\A$. However, in ISA, the local adjacencies may not be as so straightforward. $\B_{\Sb,\Sb}^{-1}$ is the Schur complement of $[d] \backslash \Sb$ block in $\I - \A$. Typically, $\B_{\Sb,\Sb}^{-1}$ does not equal $\I - \A_{\Sb,\Sb}$, and $\I - \A_{\Sb,\Sb}$ is not an ISA solution either:\looseness=-1

\begin{example}\label{example:BSS_inv_vs_ASS}

    In~\cref{fig:isa_ling_examples}(ii), consider $\Sb=(1,3,4)$, i.e., a target $T=4$ and its $\mb_\G(T)=\{1,3\}$. The ISA characterization $\B_{\Sb,\Sb}^{-1}$ separates $\X_\Sb$ into three independent irreducible subspaces ($1$-dim components):\looseness=-1  
    $$\B_{\Sb,\Sb}^{-1}\X_\Sb=\begin{bmatrix}1&0&0 \\ -ac&1&0\\ -b&-d&1\end{bmatrix}\begin{bmatrix}X_1 \\ X_3\\ X_4\end{bmatrix}=\begin{bNiceArray}{c}[margin]
    E_1 \\ cE_2+E_3\\ E_4\CodeAfter \SubMatrix{.}{1-1}{1-1}{\}}[xshift=3mm]\SubMatrix{.}{2-1}{2-1}{\}}[xshift=3mm]\SubMatrix{.}{3-1}{3-1}{\}}[xshift=3mm]
\end{bNiceArray} \ ,\\[-0.2em]$$
    but the local adjacencies $\I-\A_{\Sb,\Sb} \neq \B_{\Sb,\Sb}^{-1}$, and by
    $$(\I-\A_{\Sb,\Sb})\X_\Sb=\begin{bmatrix}1&0&0 \\ 0&1&0\\ -b&-d&1\end{bmatrix}\begin{bmatrix}X_1 \\ X_3\\ X_4\end{bmatrix}=\begin{bNiceArray}{c}[margin]
    X_1 \\ X_3\\ E_4\CodeAfter \SubMatrix{.}{1-1}{2-1}{\}}[xshift=3mm]\SubMatrix{.}{3-1}{3-1}{\}}[xshift=3mm]
\end{bNiceArray} \ ,\\[-0.2em]$$
    it is not an ISA solution, as it produces only two independent subspaces, of which the first one $(X_1^\intercal,X_3^\intercal)^\intercal$ is not irreducible with a decomposition $\begin{bmatrix}1&0\\-ac&1\end{bmatrix}$.\looseness=-1
\end{example}
Write the matrix inverse in block form we will have:
\begin{equation}\label{eq:BSS_inv_outside}
\I - \B_{\Sb,\Sb}^{-1} = \A_{\Sb,\Sb} + \A_{\Sb,\bar{\Sb}}(\I-\A_{\bar{\Sb},\bar{\Sb}})^{-1}\A_{\bar{\Sb},\Sb},
\end{equation}
where $\bar{\Sb}\coloneqq\mathcal{V}\backslash\Sb$. By~\cref{eq:BSS_inv_outside}, the $(i,j)$-th entry of $\I - \B_{\Sb,\Sb}^{-1}$ corresponds not only to the direct causal effect from $j$ to $i$, but also the total causal effect from $j$ to $i$ through all other variables outside of $\Sb$.

With this in mind, we note an issue on diagonals: while the global demixing matrix $\B^{-1}=\I-\A$ always has diagonal ones as we assume no self-loops, it may not be the case locally. Specifically, if $\G$ is acyclic, $\B_{\Sb,\Sb}^{-1}$ still has diagonal ones, but this does not hold for cyclic $\G$:\looseness=-1

\begin{example}\label{example:BSS_inv_no_diagonal_ones}
    Consider the LiNG SEM in~\cref{fig:isa_ling_examples}(iii).\looseness=-1
    $$\A=\begin{bmatrix}0&b&0 \\ a&0&0\\ 0&c&0\end{bmatrix}; \ \ \B=\frac{1}{1-ab}\begin{bmatrix}1&b&0 \\ a&1&0\\ ac&c&1-ab\end{bmatrix}.\\[-0.2em]$$
    Let $\Sb=(2,3)$, i.e., $T=3$ and its $\mb_\G(T)=\{2\}$,
    $$\B_{\Sb,\Sb}^{-1}=\begin{bmatrix}1-ab&0 \\ -c&1\end{bmatrix},\\[-0.2em]$$
    where the diagonal entry on $X_2$ is not one. This is because $X_2$ is on a cycle outside of $\Sb$, which, from the local view of $\Sb$, is equivalent to a self-loop on $X_2$. The strength of this ``self-loop'' is thus unidentifiable.
\end{example}
\subsection{Local LiNG Identification from ISA}\label{subsec:local_ling_identification_from_isa}
Having defined a specific ISA characterization $\B_{\Sb,\Sb}^{-1}$, we are now left with the task to post-process any general ISA solution to this specific one (and its equivalence class, if any). Recall that in the global ICA case, the adjacency matrix equivalence class $\mathcal{B}$ directly stems from row permutations on any demixing matrix $\W$. However, with ISA, we face more complex cross-rows indeterminacies (\cref{thm:indeterminacies_of_isa}). How to reduce them? Moreover, even if $\B_{\Sb,\Sb}^{-1}$ is exactly recovered, a challenge lies still in translating it back into LiNG model parameters, as it may not directly represent adjacencies and may be unidentifiable due to external paths (\cref{example:BSS_inv_vs_ASS,example:BSS_inv_no_diagonal_ones}). Then, what is identifiable locally, and how? We address these questions below.\looseness=-1

Consider a target vertex $T$ and its $\mb_\G(T)$. Let $\Sb$ be $\{T\}\cup \mb_\G(T)$ with $m$ elements. W.l.o.g., we rename vertices s.t. $\Sb$ reads $1$ to $m$, i.e., $\Sb=[m]\subset[d]=\mathcal{V}$. Assume a LiNG SEM $\X=\A\X+\E=\B\E$. Perform ISA on $\X_\Sb$, we obtain a solution $\W$ and the associated subspace partition $\PRT_\W$. By~\cref{thm:indeterminacies_of_isa,thm:BSS_inv_is_ISA}, $\W$ can be row-permuted and subspace-scaled into $\B_{\Sb,\Sb}^{-1}$.\looseness=-1

\subsubsection{Admissible Permutations}
So the first step is to ``re-permute'' $\W$. Since columns of $\W$ correspond exactly to $X_1$ through $X_m$, rows permutation of $\W$ can be seen as assigning names to each row, thereby forming their one-to-one correspondence to $X_1$ through $X_m$ also. Intuitively, rows within a same multi-dim subspace always correspond to variables with common hidden confounders and are thus mutually unidentifiable. However, different subspaces collectively, especially singleton subspaces (components), should be re-identified to their correct locations. Nonetheless, we first note that the nonzero-diagonal permutation as in ICA, is incorrect here:\looseness=-1

\begin{figure}[t]
\centering
\includegraphics[width=0.37\textwidth]{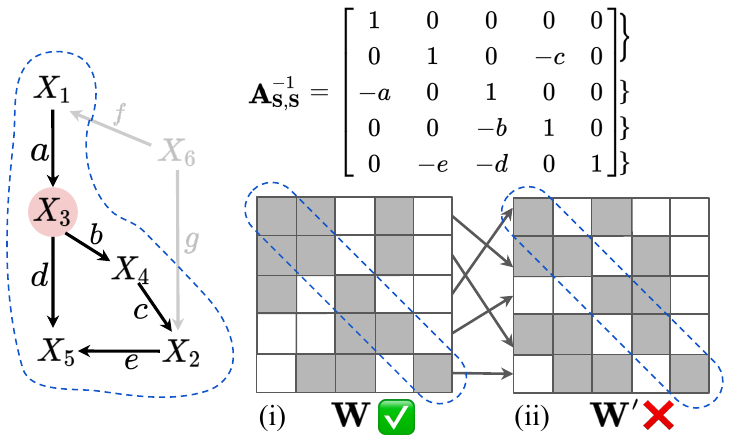}
\caption{For~\cref{example:permutation_diagonal_nonzeros_on_isa}: row permutation on ISA matrices with nonzero diagonals can be entirely incorrect.\looseness=-1}
\label{fig:permutation_diagonal_nonzeros_on_isa}
\end{figure}

\begin{example}\label{example:permutation_diagonal_nonzeros_on_isa}
    Consider an acyclic $\G$ as in~\cref{fig:permutation_diagonal_nonzeros_on_isa}. Let $\Sb$ be $(1,2,3,4,5)$, i.e., a target $T=3$ and its $\mb_\G(T)$. The true but unknown $\B_{\Sb,\Sb}^{-1}$ is provided for reference. We are only given an ISA output $\W$, as in (i), and its $\PRT_\W=\{(1,2),(3),(4),(5)\}$. Actually, $\W$ is just scaled from $\B_{\Sb,\Sb}^{-1}$ without permutation, though this is unknown. Comparing $\W$ to $\B_{\Sb,\Sb}^{-1}$, we notice that within the subspace of the 1st and 2nd rows, the nonzero entries are mixed by the general scaling. If we were to still follow the ``admissible'' criteria of nonzero diagonals as in ICA, we see that $\W$ is already satisfied (and is indeed correct). However, another permutation $\W'$ as in (ii), is also satisfied but is entirely incorrect (even on singletons' locations), leading to incorrect edges like $2\rightarrow 4$, $3\rightarrow 1$.\looseness=-1

    Why does incorrect permutation (ii) occur? Note that $\B_{\Sb,\Sb}^{-1}$ possesses a unique row permutation (itself) with nonzero diagonals, so the blame falls on the scaling to 1st and 2nd rows with more nonzeros. Fortunately, these spurious nonzeros reveal themselves via rank deficiency. In (ii), even though nonzero diagonals exist, the diagonal block of the first subspace, $\W'_{(2,4),(2,4)}$, is proportional to $[1, \ -c]$ and has rank 1. Inspired by this, we can eliminate spurious nonzeros by forcing not the nonzero diagonal entries as in ICA, but the invertible diagonal blocks, formally described below.\looseness=-1
\end{example}
\begin{definition}\label{def:admissible_block_permutations}
    Given an ISA solution $\W$ and the associated partition $\PRT_\W$, a permutation $\pi$ is called \textit{admissible}, if $\forall \Sb_i \in \PRT_\W$, $\rank((\Pb^\pi \W)_{\pi[\Sb_i], \pi[\Sb_i]})=|\Sb_i|$.\looseness=-1
\end{definition}
Admissible permutations defined in~\cref{def:admissible_block_permutations} are ``sound and complete''. See~\cref{proof:thm_correct_isa_ling} for formal definition and proof. Roughly speaking, all such admissible rows permutations correspond exactly to all rows permutations on $\B_{\Sb,\Sb}^{-1}$ with nonzero diagonals (viewing each subspace collectively), which then correspond exactly to the LiNG equivalence class on $\Sb$.

\subsubsection{Identifiable Local Causal Effects}
Having admissible permutations, now we proceed to identify local causal structures and coefficients. As demonstrated in~\cref{example:BSS_inv_vs_ASS,example:BSS_inv_no_diagonal_ones}, ISA matrices is not overall reliable. However, note that the misidentification of an $X_i$ on both examples can be attributed to an incoming path (either in a cycle or not) outside of $\Sb$. This immediately sparks us that if all of $i$'s parents are included in $\Sb$, such issues should not arise:

\begin{restatable}{lemma}{LEMMAWHENPARENTSAREINISA}\label{lemma:when_parents_are_in_ISA}
    Given an ISA solution $\W$ and $\PRT_\W$ on $\X_\Sb$, $\forall i\in \Sb$, if $\pa_\G(i)\subset \Sb$, then its exogenous noise component $E_i$ is separated out, i.e., $\exists j\in [m]$ s.t. $(j)\in \PRT_\W$ and $(\W\X_\Sb)_j = cE_i$ with a scaling factor $c$. Moreover, the incoming causal strengths to $X_i$ are identified up to $c$, i.e., the row vector $\W_j = c (\I - \A_{\Sb,\Sb})_i$.
\end{restatable}
\cref{lemma:when_parents_are_in_ISA} becomes especially helpful in our case: by definition of MB, for any of $T$ and its children, all its parents are included in $\{T\}\cup \mb_\G(T)$, thus blocking
all confounding paths, enabling recovery of exogenous noise, and moreover, the exact causal strengths. Once all edges into $T$ and $T$'s children are identified, we've attained the goal of local causal discovery, as these edges include all edges to and from $T$. As for other variables in MB, e.g., parents and spouses, they may be entangled within subspaces and remain unidentifiable, but this does not pose a concern anymore.\looseness=-1

By~\cref{lemma:when_parents_are_in_ISA}, $T$ and its children produce independent components ($1$-dim subspaces) by ISA. But conversely, an unconfounded parent or spouse can also produce a $1$-dim subspace. Then, which of these components correspond exactly to our main focus, $T$ and its children? The answer can be read off of $T$'s column in $\W$:

\begin{restatable}{lemma}{LEMMATCOLUMNASTANDCHILDREN}\label{lemma:T_column_as_T_and_children}
    Given an ISA solution $\W$ and $\PRT_\W$ on $\X_\Sb$ for $\Sb=\{T\}\cup \mb_\G(T)$. Denote by $\mathbf{C}\coloneqq \operatorname{supp}(\W_{:,T}) = (i\in[m]: \W_{i,T}\neq 0)$. Then $\forall i\in\mathbf{C}$, $\W_i$ must produce a single component, i.e., $(i)\in\PRT_\W$. Moreover, $\{\pi[\C]: \pi \text{ admissible to }\W \} = \{\operatorname{supp}(\A'_{:,T}):\A' \in \mathcal{B}\}$.
\end{restatable}
In essence, \cref{lemma:T_column_as_T_and_children} interprets the nonzero row indices on $T$'s column vector in $\W$ as $T$ and $T$'s children. Note that there can be multiple directed graphs in the LiNG equivalence class, leaving different choices of $\Sb$ subsets as $T$'s children. Any such choice can be interpreted by an admissible row permutation, and vice versa.

\subsubsection{The Local ISA-LiNG Algorithm}
Finally, we have the local ISA-LiNG~\cref{alg:local_isa_ling}. Below we give an illustrative example on how it works:

\begin{algorithm}[t]
	\caption{Local ISA-LiNG}
	\label{alg:local_isa_ling}
	\hspace*{0.02in} {\bf Input:}
	A target $T\in\mathcal{V}$, its oracle MB $\mb_\G(T)$, and data $\X$. Assume w.l.o.g. $\Sb\coloneqq\{T\}\cup\mb_\G(T)=[m]$\\
	\hspace*{0.02in} {\bf Output:}
	A set of directed weighted edge sets

 \begin{algorithmic}[1]
	\STATE Initialize the output set $\mathcal{K} \coloneqq \emptyset$;
    \STATE Obtain an ISA solution $\W$ with $\PRT_\W$ on $\X_\Sb$;
    \STATE Set $\mathbf{C}\coloneqq \operatorname{supp}(\W_{:,T}) = (i\in[m]: \W_{i,T}\neq 0)$;
    \FOR{any permutation $\pi$ admissible to $\W,\PRT_\W$}
        \STATE Initialize $\K\coloneqq \emptyset$;
        \STATE Set $\W' \coloneqq \Pb^\pi \W$;
        \STATE Set scaling matrix $\D$: $\forall \Sb_i \in \PRT_\W$, $\D_{\pi[\Sb_i],\pi[\Sb_i]} \coloneqq (\W'_{\pi[\Sb_i],\pi[\Sb_i]})^{-1}$ and $\D_{\pi[\Sb_i],[m] \backslash \pi[\Sb_i]} \coloneqq \mathbf{0}$;
        \STATE Set $\A' \coloneqq \I - \D \W'$;

        \FOR{$i\in\mathbf{C}$}
            \STATE \textbf{Assert} $(i) \in \PRT_\W$;
            \STATE Add to $\K$ a weighted edge denoted as $(j\rightarrow \pi[i], \A'_{\pi[i],j})$, for each $j\in[m]$ with $\A'_{\pi[i],j} \neq 0$;\looseness=-1
        \ENDFOR
        \STATE Set $\mathcal{K}\coloneqq \mathcal{K} \cup \{\K\}$;
    \ENDFOR
	\STATE {\bfseries Return} $\mathcal{K}$;
	\end{algorithmic}
\end{algorithm}

\begin{example}\label{example:isa_ling_procedure}
    Consider the example in~\cref{fig:isa_algorithm_example}. There are two graphs in the global equivalence class $\mathcal{B}$, as shown in the upper row. Let $\Sb$ be $(1,2,3,4,5)$, i.e., a target $T=3$ and its $\mb_\G(T)$. An ISA on $\X_\Sb$ gives $\W$ with nonzero patterns as in the lower left matrix, where specifically, the striped entries are nonzero but rank deficient (see~\cref{def:admissible_block_permutations}). The 3rd ($T$-th) column has three nonzero entries, corresponding to $T$ and its two children, which are yet unknown and can be different in different equivalent graphs. Two admissible rows permutations (the lower row) reveal their variable correspondences, with all edges into $T$ and its children (dark edges in the graphs) recovered correctly for all equivalent graphs with different global directed cycles, while note these local edges themselves are acyclic.\hspace{-2em}\looseness=-1
\end{example}

\begin{figure}[t]
\centering
\includegraphics[width=0.31\textwidth]{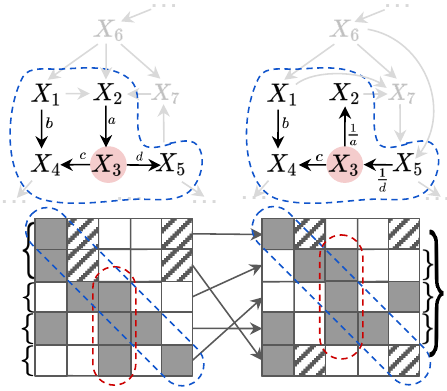}
\caption{For~\cref{example:isa_ling_procedure}, to illustrate~\cref{alg:local_isa_ling}.\looseness=-1}
\label{fig:isa_algorithm_example}
\vspace{-0.7em}
\end{figure}

\begin{restatable}[Correctness of local ISA-LiNG]{theorem}{THMCORRECTNESSOFLOCALISALING}\label{thm:correct_isa_ling}
    For any $T\in\mathcal{V}$, let $\mathcal{K}$ be set of weighted edge sets returned by~\cref{alg:local_isa_ling} on $T$, $\mb_\G(T)$, and $\X$. We have:\begin{align*}
    \mathcal{K} = \{ & \{(i\rightarrow j, \A'_{j,i}) : \forall j \in \{T\} \cup \ch_{\G'}(T), \forall i \in \pa_{\G'}(j)\} : \\ 
            & \forall \A'\in\mathcal{B}, \text{ and the graph } \mathcal{G}' \text{ defined by } \A'\}.
\end{align*}
\end{restatable}
\vspace{-0.2em}
The local ISA-LiNG algorithm (\cref{alg:local_isa_ling}) correctly identifies all the causal effects into the target $T$ and all its children, for all LiNG models that equivalently entails the distribution of $\X$. That is, with only local variables, we identify exactly what can be identified globally. Note that this identification is unique (i.e., the returned $\mathcal{K}$ consists of a single item) if and only if none of $T$ and $\ch_{\G}(T)$ is part of any cycles in $\G$ (including the case where $\G$ is acyclic).

\subsection{With the Notion of Stability}

The $\mathcal{B}$ defined in~\cref{def:LiNG_equiv_class} characterizes the entire global LiNG equivalence class, yet not all models within it are ``stable''. In dynamical systems, stability refers to ``the dissipation of the effects of one-time noise in models''~\citep{lacerda2008discovering}. Applied to causal models, a model is ``stable'' when any infinitely long path (after traversing loops) result in zero causal effect. For example, in a simple cycle with two variables and two edges both carrying weights $2$, the model is unstable, with the cycle product of $4>1$ leading to explosion. When both weights are $0.5$, the model remains LiNG equivalent to the former one but achieves stability, with the cycle product of $0.25<1$. Formally, a global LiNG model is said to be \textit{stable} when its adjacency matrix $\A$ is convergent, i.e., $\lim_{t\rightarrow \infty} \A^t = \mathbf{0}$. Note that here the entry $(\A^t)_{i,j}$ represents the summed causal effect from $j$ to $i$ along all paths of length $t$.

In practical global causal discovery scenarios, an often-made assumption is the stability of the underlying LiNG model, and people usually focus on identifying the stable LiNG model(s), instead of the entire equivalence class $\mathcal{B}$. This is straightforward in ICA-LiNG~\citep{lacerda2008discovering}: as the entire $\mathcal{B}$ can be recovered first, we then only need to check the convergence of each item within $\mathcal{B}$. However, when we only have local variables, can we still recover the local part corresponding to the global stable model(s)?

The answer is affirmative but with constraints: our method can still correctly find stable solutions locally, as long as this local stable solution is identifiable. Denote the stable sub-equivalence class as $\mathcal{B}^* \coloneqq \{\A \in \mathcal{B}: \lim_{t\rightarrow \infty} \A^t = \mathbf{0}\}$. When cycles in the ground-truth $\G$ are disjoint, there exists a unique global stable model, i.e., $|\mathcal{B}^*| = 1$. Let $\A^*$ be this unique stable adjacency matrix, and $\G^*$ be the corresponding graph.
In this case, simply by adhering to an additional \textit{local stability} condition, the stable solution can be identified locally:\looseness=-1

\begin{restatable}[Identifying stable solutions locally, with disjoint cycles]{corollary}{COROLOCALSTABLE}\label{coro:local_stable}
    Suppose the cycles are disjoint in $\G$. Consider a modified version of~\cref{alg:local_isa_ling} in which the line 
        ``\textbf{if} $\A'$ is not convergent: \textbf{skip}''
    is added between lines 8 and 9. Then, this modified version of~\cref{alg:local_isa_ling} will yield a single local model, corresponding exactly to the unique global stable model. That is, the returned $\mathcal{K}$ consists of a single item $\K$, with\looseness=-1
    $$\K = \{(i\rightarrow j, \A^*_{j,i}) : \forall j \in \{T\} \cup \ch_{\G^*}(T), \forall i \in \pa_{\G^*}(j)\}.$$
\end{restatable}
\vspace{-0.4em}
\cref{coro:local_stable} is valid as here stability is determined sufficiently and necessarily by the cycle products, which is preserved locally. However, when some cycles in $\G$ intersect, the situation becomes more complex. Globally, there may be none or multiple global stable models in $\mathcal{B}^*$. Locally, while in this case, our \cref{alg:local_isa_ling} can still identify local correspondences of all equivalent solutions (\cref{thm:correct_isa_ling}), the exact identification of the global stable solutions from local variables alone becomes inherently impossible. Intuitively, this is because that external cycles appear as self-loops on the local variables. More details are in~\cref{sec:proofs}.\looseness=-1

\section{REGRESSION-BASED VARIANT}\label{sec:inv_direct_lingam}
In~\cref{sec:isa_ling} we propose a local ISA-based method suitable for both acyclic and cyclic graphs. In this section, we focus on the specific scenario where there are no cycles in $\G$, i.e., $\X$ follows a linear non-Gaussian \textit{acyclic} model (LiNGAM~\citep{shimizu2006lingam}), and propose an alternative local regression-based method. The relationship between this section and~\cref{sec:isa_ling} can be likened to that of Direct-LiNGAM~\citep{shimizu2011directlingam} and ICA-LiNG~\citep{lacerda2008discovering} in the global context, with the former utilizing non-Gaussianity by ICA, and the latter by Darmois-Skitovitch theorem~\citep{darmois1953analyse,skitovitch1953property}.\looseness=-1

Acyclicity renders the existence of a \textit{causal ordering}, i.e., vertices $\mathcal{V}$ can be ordered so that no later vertex has a direct edge onto any earlier variable. When all the variables in $\X$ are involved, namely, with \textit{causal sufficiency}, \citet{shimizu2011directlingam} gives the method named Direct-LiNGAM to uniquely identify the DAG $\G$ by estimating its causal ordering: Regress $X_j$ on $X_i$, if the residual is statistically independent with the regressor $X_i$, then $i$ is causally earlier than $j$. If such an independence holds for $X_i$ on all its pairwise regressions with the remaining $X_j$s, $i$ must be a \textit{root} vertex. Subroots are then recursively identified in a same way, forming the causal ordering.\looseness=-1

However, when only a subset of variables (as of here, $\{T\}\cup\mb_\G(T)$) is involved, Direct-LiNGAM does not work, as causal sufficiency is violated, and there can be no independent residual due to hidden confounders, just like the absent independent components in ICA.\looseness=-1
\begin{example}\label{example:hidden_confounder}
    In~\cref{fig:isa_ling_examples}(i), with a confounder $X_1$ outside of $T$'s MB, neither regressing $X_2$ on $X_3$ nor the converse results in independent residuals, making it impossible to identify any ``local root'' in $\mb_\G(T)$.
\end{example}
\vspace{-0.2em}
While identifying ``local roots'' is impossible, can we reverse our perspective from top-down to bottom-up and identify ``local leaves'' instead? Interestingly, the answer seems affirmative: In~\cref{example:hidden_confounder}, regressing $X_4$ on $\{X_2,X_3\}$, the residual is exactly the exogenous noise $E_4$ and is independent to $\{X_2,X_3\}$. Formally, for any vertex subset $\Sb\subset \mathcal{V}$, we denote the corresponding random vector as $\X_\Sb\coloneqq [X_i:i\in\Sb]^\intercal$. Perform ordinary least square error linear regression of a random variable $X_i$ on a random vector $\X_\Sb$, the asymptotic coefficients of fit is $\beta_{\Sb\rightarrow i}\coloneqq\operatorname{cov}(\X_\Sb, \X_\Sb)^{-1} \operatorname{cov}(\X_\Sb,X_i)$, where for $j\in\Sb$, $\beta_{\Sb\rightarrow i}^j$ is the coefficient on $X_j$. Denote the regression residual as $R_{\Sb\rightarrow i} \coloneqq X_i - \beta_{\Sb\rightarrow i}^\intercal \X_\Sb$. Denote $i$'s descendants in $\G$ as $\des_\G(i)$. We have:\looseness=-1

\begin{restatable}{lemma}{LEMMAINDEPENDENTNOISE}\label{lemma:independent_noise}
        For any $i\in\mathcal{V},\Sb\subset \mathcal{V}\backslash\{i\}$, if $R_{\Sb\rightarrow i} \indep \X_\Sb$, i.e., independent residual, then $\Sb \cap \des_\G(i) = \emptyset$.\looseness=-1
\end{restatable}
\vspace{-0.2em}
\cref{lemma:independent_noise} generalizes regressions in~\citep{shimizu2011directlingam} from single variables to multi-dim vectors, but with a same idea: independent residuals imply causal ordering. While as in~\cref{example:hidden_confounder}, independent residuals may be absent for ``local roots'' due to confounders (echoed as multi-dim subspaces in ISA), they must exist for ``local leaves'' (echoed as the $1$-dim components in ISA). This is because, again, as in~\cref{lemma:when_parents_are_in_ISA}, that parents of $T$ and its children are included in the MB, thus blocking all confounding paths, enabling recovery of exogenous noise and the exact causal strengths:\looseness=-1

\begin{restatable}{lemma}{LEMMACORRECTREGRESSIONCOEFS}\label{lemma:correct_regression_coefs}
    $\forall i,\Sb$ in $\mathcal{V}$, if $\pa_\G(i)\subset \Sb \subset \mathcal{V} \backslash \des_\G(i)$, then $\forall j\in\Sb$, $\beta_{\Sb\rightarrow i}^j = \A_{i,j}$, and $R_{\Sb\rightarrow i} = E_i$ (so $\indep \X_\Sb$).\looseness=-1
\end{restatable}
\vspace{-0.2em}
\cref{lemma:correct_regression_coefs} holds generally for linear acyclic SEMs, echoing the local Markov property: given all its parents, a variable is independent of other non-descendants, enabling accurate estimation of direct effects to it. \cref{lemma:independent_noise,lemma:correct_regression_coefs} then readily leads to~\cref{alg:inverse_direct_lingam}.\looseness=-1

\begin{algorithm}[t] %
	\caption{Inverse Direct-LiNGAM}
	\label{alg:inverse_direct_lingam}
	\hspace*{0.02in} {\bf Input:}
	A target vertex $T\in\mathcal{V}$, its oracle MB $\mb_\G(T)$, and their corresponding variables in $\X$\\
	\hspace*{0.02in} {\bf Output:}
	A set of directed edges with weights

 \begin{algorithmic}[1]
	\STATE Initialize the remaining vertex set $\U\coloneqq \{T\}\cup \mb_\G(T)$, and the output edge set $\K \coloneqq \emptyset$;
	\WHILE{$\U \neq \emptyset$}
      \IF{$\exists k\in \U\backslash\{T\}$ s.t. $\beta_{\U\backslash\{k\}\rightarrow k}^T = 0$}
        \STATE Set $\U\coloneqq \U\backslash\{k\}$;
        \STATE \textbf{continue} to line 2;
      \ENDIF
      \STATE \textbf{Assert} $\exists j\in \U$ s.t. $R_{\U\backslash\{j\}\rightarrow j} \indep \X_{\U\backslash\{j\}}$;
      \STATE Set $j$ as any one found in line 7;
      \STATE Add to $\K$ an edge $i\rightarrow j$ with weight $\beta_{\U\backslash\{j\}\rightarrow j}^i$ for each $i\in\U\backslash\{j\}$ with $\beta_{\U\backslash\{j\}\rightarrow j}^i \neq 0$;
      \STATE \textbf{break} if $j=T$; Otherwise set $\U\coloneqq \U\backslash\{j\}$;
    \ENDWHILE
	\STATE {\bfseries Return} $\K$;
	\end{algorithmic}
\end{algorithm}

The basic idea of~\cref{alg:inverse_direct_lingam} is to recursively identify ``local leaves'', i.e., all $T$'s children until $T$, via an ``inverse causal ordering''. Independent residuals must exist for these variables, as all their parents are included locally (lines 7-9). Lines 3-6 serve to avoid errors due to spouses, which could be hidden confounded:\looseness=-1

\begin{example}\label{example:spouses_correction_inv1}
    In~\cref{fig:inv_direct_lingam_examples}(i), $T=1$, $\mb_\G(T)=\{2,3,4,5\}$. After $X_5$ is first identified and removed, the remaining two ``last leaves'' $X_3,X_4$ are confounded by $X_6$ hidden outside of $\mb_\G(T)$, and thus neither can produce independent residual. The iteration cannot proceed, unless these two spouses are removed.\looseness=-1
\end{example}
Even without confounders, edges produced by independent residuals may still be incorrect due to spouses:\looseness=-1
\begin{example}\label{example:spouses_correction_inv2}
    In~\cref{fig:inv_direct_lingam_examples}(ii), $T=1$, $\mb_\G(T)=\{2,4,5\}$. After $X_5$ is first identified and removed, though the ``last leaf'' $X_4$ produces independent residual regressing on $X_1,X_2$, due to hidden $X_3$, the coefficient on $X_1$ is nonzero, yielding an incorrect edge $1\rightarrow 4$. Correction requires removing this spouse $X_4$.\looseness=-1
\end{example}
\vspace{-0.2em}
\begin{figure}[t]
\centering
\includegraphics[width=0.40\textwidth]{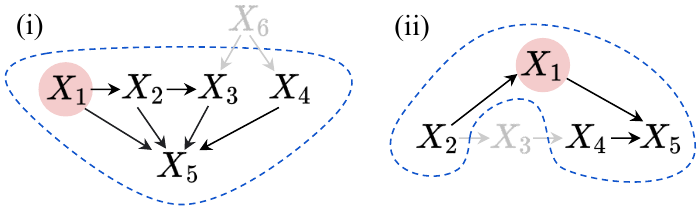}
\caption{Examples to illustrate~\cref{alg:inverse_direct_lingam}.}
\label{fig:inv_direct_lingam_examples}
\vspace{-0.6em}
\end{figure}

With spouses corrected, \cref{alg:inverse_direct_lingam} accurately estimate all edges into $T$ and its children, including edges adjacent to $T$ (the purpose of local search). Formally,\looseness=-1

\begin{restatable}[Correctness of~\cref{alg:inverse_direct_lingam}]{theorem}{THMCORRECTNESSOFINVDIRECTLINGAM}\label{thm:inv_direct_lingam_correct}
    For any $T\in\mathcal{V}$, let $\K$ be the weighted edge set returned by~\cref{alg:inverse_direct_lingam} on $T$, $\mb_\G(T)$, and $\X$. We have:
    $$\K = \{(i\rightarrow j, \A_{j,i}) : \forall j \in \{T\} \cup \ch_\G(T), \forall i \in \pa_\G(j)\}.$$
\end{restatable}
\vspace{-0.4em}
\cref{thm:inv_direct_lingam_correct} is similar to~\cref{thm:correct_isa_ling}, except that the DAG can be uniquely identified. See~\cref{proof:thm_inv_direct_lingam_correct} for the proof, and~\cref{sec:acyclic_postprocessing} for also an alternative postprocessing of ISA, with the same ``ordering'' idea here.\looseness=-1

\vspace{-0.1em}
\section{EXPERIMENTS}\label{sec:experiments}
We assess the effectiveness of our method for cyclic and acyclic cases in \cref{sec:cyclic_experiments,sec:acyclic_experiments}, respectively. We provide an analysis of how our method performs under different sample sizes in \cref{sec:analysis_sample_sizes}, and an experiment on real data in \cref{sec:real_data}. The implementation details and running times are discussed in \cref{sec:supplementary_experiment_details,sec:running_time}, respectively.

\subsection{Cyclic Case}\label{sec:cyclic_experiments}
We conduct experiments to illustrate the output of our method, by adopting the left cyclic graph in \cref{fig:isa_algorithm_example} as ground truth. We simulate $2000$ samples from the LiNG SEM in \cref{eq:linear_sem_adj}, of which the nonzero weights of $\A$ are sampled uniformly from $[-0.9,-0.5]\cup[0.5, 0.9]$, and each exogenous noise $E_i$ is sampled uniformly from $[-c_i,c_i]$ to the power of $5$, with $c_i$ sampled randomly from $[0.75,1.25]$.

We first run the ICA-LiNG method by \citet{lacerda2008discovering} on all variables. To perform local causal discovery, we also run our Local ISA-LiNG method on target $T=3$ and its MB $\{1,2,4,5\}$. An example of the outputs by both methods, including the estimated edge weights, are provided in \cref{fig:cyclic_graph_outputs} in \cref{sec:supplementary_experiment_figures}. One observes that our method correct identifies the edges according to \cref{thm:correct_isa_ling}, and that the estimated edge weights are close to the true ones.

\textbf{With stability.} \ \  
To conduct a systematic validation, we restrict the cycles in the true graphs to be disjoint and the true $B$ matrices to be stable using an accept-reject approach; that is, the spectral radius of $B$ has to be strictly smaller than one. In this case, \cref{coro:local_stable} indicates that the stable solution can be uniquely identified locally. We simulate $50$-node directed cyclic graphs (DCGs) with maximum degree of $4$, and $2000$ samples from the LiNG SEM in \cref{eq:linear_sem_adj}. We use the same setup described above for the edge weights and noise distributions. To perform local causal discovery, we randomly select a target $T$ that is part of a cycle in the $50$-node DCGs. Due to the lack of local causal discovery baselines that handle cyclic graphs, we compare our method with those for acyclic cases, including GSBN \citep{margaritis1999bayesian}, Local A* \citep{ng2021reliable}, CMB \citep{gao2015local}, and LDECC \citep{gupta2023local}. Note that GSBN and Local A* require information of two-step MBs (i.e., $\mb_\G(T)$ and MB of each variable in $\mb_\G(T)$), which are not directly comparable to our method that requires only $\mb_\G(T)$; thus, we consider modifications of these methods, described in Appendix \ref{sec:supplementary_experiment_details}. We report the structural Hamming distance (SHD) of local DCG, which is explained in details in Appendix \ref{sec:metrics}.\looseness=-1

We provide the results for the methods using estimated MB in \cref{fig:lasso_mb_dag_shd_cyclic}, and using oracle MB in \cref{fig:oracle_mb_dag_shd_cyclic} in Appendix \ref{sec:supplementary_experiment_figures}. It is observed that our method achieves much lower SHD in both settings, thereby demonstrating its effectiveness for identifying the local structure.

\begin{figure*}[!t]
\hspace{-0.15em}
\begin{minipage}{0.33\textwidth}
\centering
\includegraphics[width=0.99\linewidth]{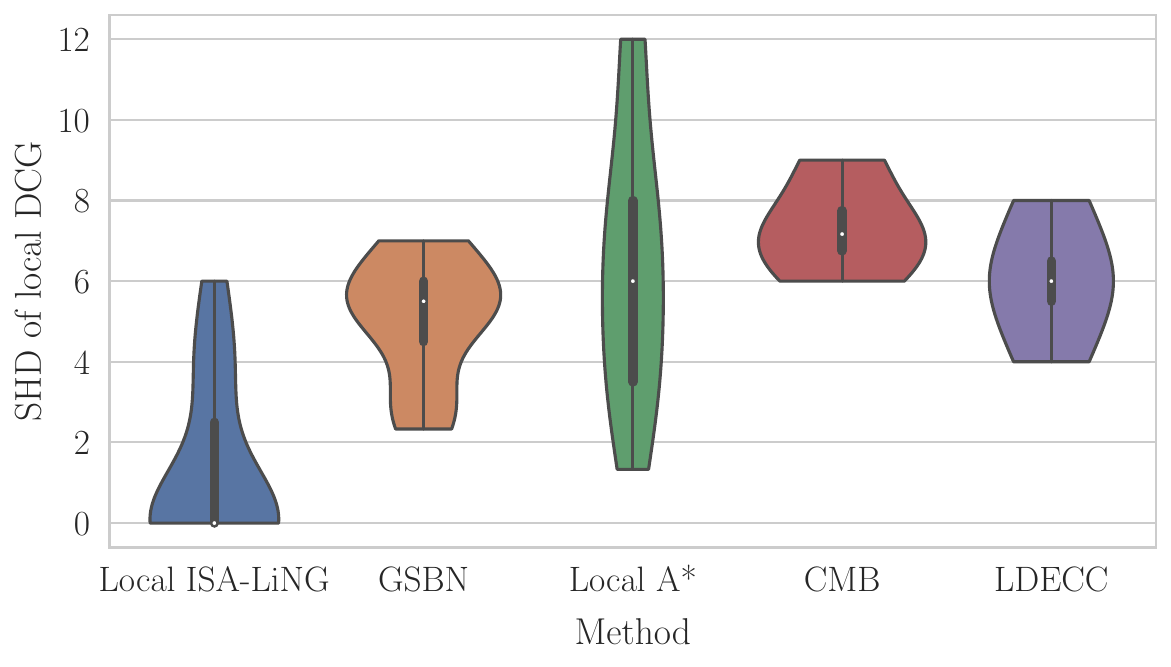}
\caption{{\small SHD of local DCG under estimated MB.}}
\label{fig:lasso_mb_dag_shd_cyclic}
\end{minipage}
\hspace{0.15em}
\begin{minipage}{0.33\textwidth}
\centering
\includegraphics[width=0.99\linewidth]{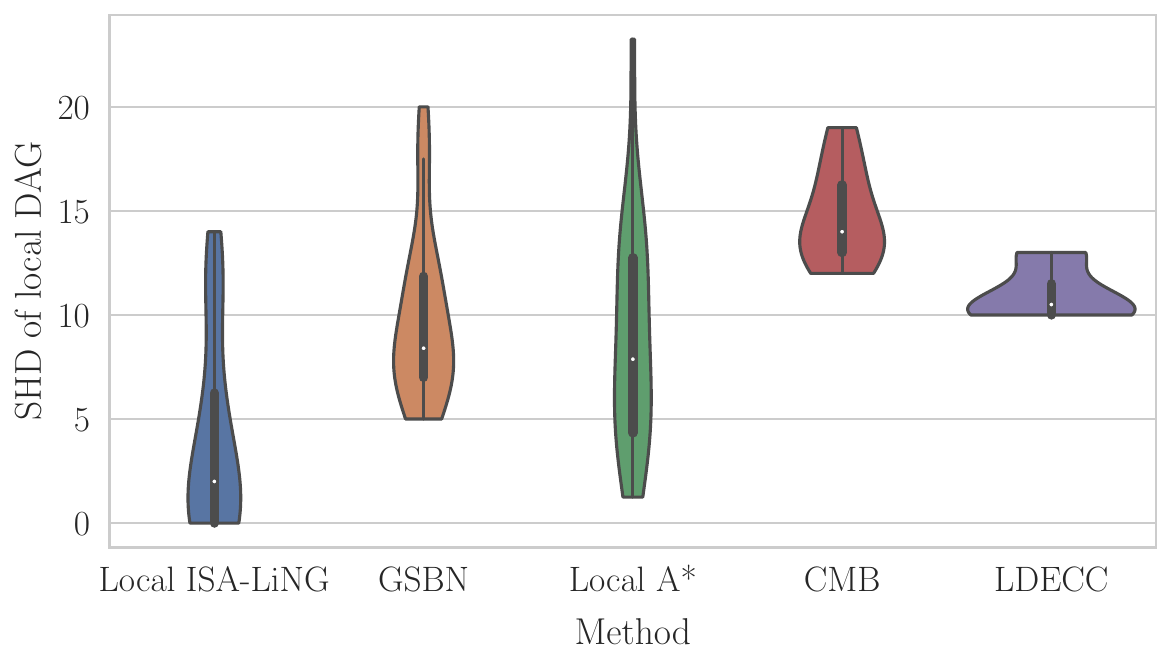}
\caption{{\small SHD of local DAG under estimated MB.}}
\label{fig:lasso_mb_dag_shd_degree_3_acyclic}
\end{minipage}
\hspace{0.15em}
\begin{minipage}{0.3\textwidth}
\centering
\includegraphics[width=0.99\linewidth]{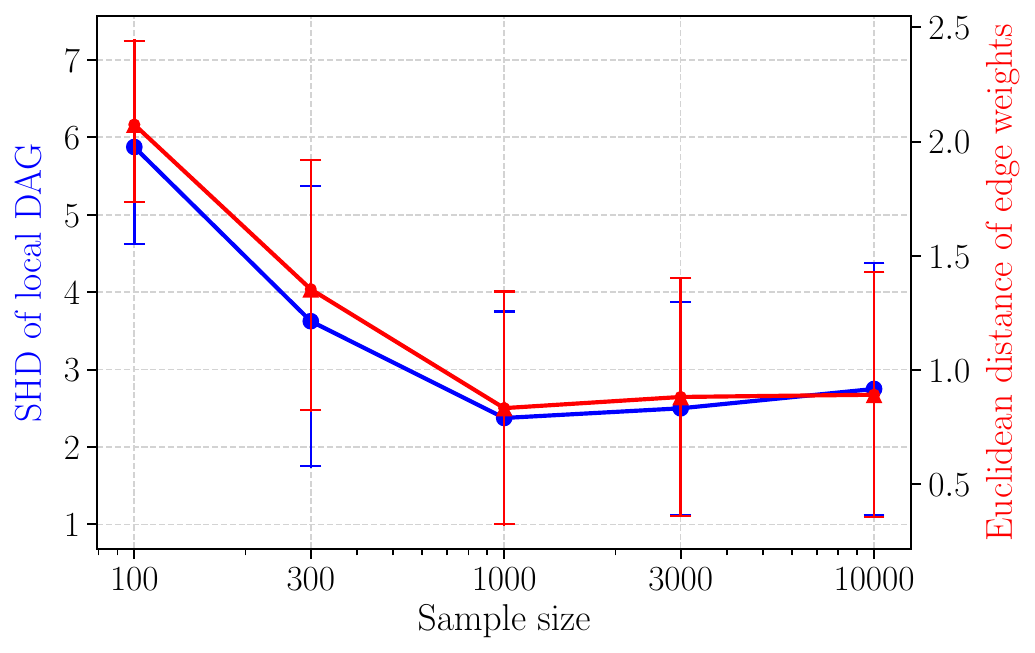}
\caption{{\small Local ISA-LiNG under oracle MB. X-axis is in log scale.}}
\label{fig:oracle_mb_different_sample_sizes}
\end{minipage}
\vspace{-0.6em}
\end{figure*}

\subsection{Acyclic Case}\label{sec:acyclic_experiments}
We consider the acyclic setting where the ground truths are DAGs.
In the acyclic case , we use a more efficient post-processing procedure for demixing matrix $\W$, described in Appendix \ref{sec:acyclic_postprocessing}. We simulate $50$-node Erd\"{o}s--R\'{e}nyi \citep{erdos1959random} DAGs, and $2000$ samples from the LiNG SEM in \cref{eq:linear_sem_adj} using the same setting (including edge weights and noise distributions) as that of \cref{sec:cyclic_experiments}.
To perform local causal discovery, we consider target $T$ from $50$-node DAGs with expected degrees of $3$ and $5$, leading to roughly $14$ and $20$ variables in the MB $\mb_\G(T)$, respectively. We report the SHD of local DAG and partially DAG (PDAG), explained in \cref{sec:metrics}.\looseness=-1

For degree of $3$, the SHDs of local DAG for the methods using estimated MB are shown in \cref{fig:lasso_mb_shd_degree_3_acyclic}, while the complete results using estimated MB and oracle MB are given in \cref{fig:lasso_mb_shd_degree_3_acyclic,fig:oracle_mb_shd_degree_3_acyclic} in \cref{sec:supplementary_experiment_figures}, respectively, due to space limit. We provide the results for degree of $5$ in \cref{fig:lasso_mb_shd_degree_5_acyclic} in \cref{sec:supplementary_experiment_figures}. Similar to the cyclic case, our method achieves much lower SHD for both local DAG and PDAG as compared to the baselines. One also observes that GSBN and Local A* performs better than CMB and LDECC.

\subsection{Analysis of Different Sample Sizes}\label{sec:analysis_sample_sizes}
We provide an analysis of the proposed method across sample sizes $n\in\{100,300,1000,3000,10000\}$, following the data generating procedure in \cref{sec:acyclic_experiments}. We report the SHD of local DAG and the Euclidean distance between the estimated edge weights and the true ones. The results using oracle MB is shown in \cref{fig:oracle_mb_different_sample_sizes}, while those using estimated MB are given in \cref{fig:lasso_mb_different_sample_sizes} in Appendix \ref{sec:supplementary_experiment_figures}. As the sample size increases, both metrics decrease to small values close to zero, which help validate the asymptotic correctness of our method in terms of both structure and parameter estimation. This also demonstrates the possibility of reliable estimation even when the sample size is rather limited.\looseness=-1

Moreover, we provide the scatter plots of the estimated and true edge weights in Figures \ref{fig:oracle_mb_edge_weights} and \ref{fig:lasso_mb_edge_weights} in \cref{sec:supplementary_experiment_figures}. For larger sample sizes, the data points are increasingly grouped onto the main diagonal, showing that the estimated weights become more accurate.

\vspace{-0.2em}
\subsection{Real Data}\label{sec:real_data}
\vspace{-0.1em}
We compare our method with GSBN and Local A$^{*}$ on a standard real-world dataset that collects continuous expression levels of proteins and phospholipids within human immunological cells \citep{sachs_data}, characterized by $853$ observational samples and a ground truth DAG with $11$ variables and $17$ edges.
Here, we select PIP2, PIP3, and Akt as target variables, and compute the SHD of local DAG  obtained by all three methods. As shown in the Table \ref{table:acc_comprs}, our method achieves lower SHD in most cases. A detailed comparison of ground-truth and estimated local causal structures can be found in Figure \ref{fig:sachs_exp} in Appendix \ref{sec:supplementary_experiment_figures}.
\begin{table}[!h]

\centering
\caption{SHD of different local causal discovery methods on real data by \citet{sachs_data}.
}
\vspace{-0.1em}
{\small
\begin{tabular}{cccc}
\toprule
\textbf{Target} & Ours& GSBN& Local A$^*$\\
\midrule
{PIP2}& $\bf{1}$& 1.5& 3.6\\
{PIP3}& $\bf{1}$& $\bf{1}$ & 4\\
{Akt}& $\bf{1}$& 1.3& 1.3\\
\bottomrule
\end{tabular}
}
\label{table:acc_comprs}
\end{table}
\vspace{-0.3em}
\section{CONCLUSION}\label{sec:conclusion}
\vspace{-0.1em}
We have expanded local causal discovery to include cyclic scenarios by generalizing the classic LiNGAM-based methods. Notably, while previous local search methods based on conditional independence tests or likelihood-based scores often fail to determine the direction of certain edges, our method leverages non-Gaussianity to enable more precise edge orientations. This leads to a more comprehensive representation of the causal graph, even in cyclic contexts. Additionally, we have established identifiability guarantees for all our proposed methods. These theoretical findings have been validated using various datasets in both synthetic and real-world settings. Future work includes characterizing the number of possible structures in the cyclic equivalence class estimated by our method.

\section*{Acknowledgements}
The authors would like to thank reviewers for their helpful comments. This material is based upon work supported by the AI Research Institutes Program funded by the National Science Foundation under AI Institute for Societal Decision Making (AI-SDM), Award No. 2229881.  The project is also partially supported by the National Institutes of Health (NIH) under Contract R01HL159805, and grants from Apple Inc., KDDI Research Inc., Quris AI, and Infinite Brain Technology.

\clearpage
\appendix

\onecolumn 
{\hsize\textwidth
    \linewidth\hsize \toptitlebar {\centering
        {\Large\bfseries Local Causal Discovery with Linear non-Gaussian Cyclic Models: Supplementary Materials \par}}
    \bottomtitlebar}

\section{MARKOV BLANKET DISCOVERY FOR CYCLIC GRAPHS}\label{sec:markov_blanket}
The local causal discovery procedures presented in Sections \ref{sec:isa_ling} and \ref{sec:inv_direct_lingam} rely on knowledge about the MB of the target variable $T$ i.e., its parents, children, and spouses. To the best of our knowledge, many existing MB estimation methods, such as those based on nonparametric conditional independence test, e.g., GSMB \citep{margaritis1999bayesian}, IAMB \citep{tsamardinos2003algorithms}, and MMMB \citep{tsamardinos2006mmhc}, focus on the Bayesian network (i.e., acyclic) setting. That is, it may not be immediately clear how their estimated MB relates to the true one $\mb_\G(T)$ in the presence of cycles, partly owing to the extra complications involved when handling cycles with conditional independence tests \citep{spirtes1994conditional}. In this section, we provide a method to estimate the MB of a variable from a linear cyclic SEM. Specifically, we build upon the method proposed by \citet{loh2014high} that, similar to methods based on conditional independence tests, makes the acyclicity assumption, and further generalize it to handle cycles.

We first define the moral graph of a directed cyclic graph the same way as that of a DAG. Specifically, the moral graph of directed graph $\G$ is an undirected graph that contains an edge between two nodes if (1) they are adjacent in $\G$, or (2) they share the same children. Clearly, the MB of a variable is simply its neighbors in the moral graph of $\G$. Here, we provide a method to estimate such moral graph, which informs us about $\mb_\G(T)$. Considering the linear SEM in \cref{eq:linear_sem_adj}, the inverse covariance matrix of the distribution of variables $\mathbf{\X}$ is given by $\mathbf{\Theta}=(\mathbf{I}-\mathbf{\A})\mathbf{\Omega}^{-1}(\mathbf{I}-\mathbf{\A})^\intercal$, where $\mathbf{\Omega}\coloneqq\operatorname{diag}(\sigma_1^2,\dots,\sigma_d^2)\coloneqq\operatorname{cov}(\mathbf{\E})$ is the covariance matrix of exogenous noise components $\mathbf{\E}$. Inspired by \citet[Assumption~1]{loh2014high} in the acyclic case, we make the following assumption in the cyclic case.
\begin{assumption}\label{assumption:inv_cov}
Let $\mathbf{\A}$ and $\mathbf{\Omega}$ be the weighted adjacency matrix and noise covariance matrix, respectively, of the linear SEM in \cref{eq:linear_sem_adj}. For every $j< i$, we have
\begin{equation}\label{eq:assumption_inv_cov}
-\sigma_j^{-2}\mathbf{\A}_{i,j}-\sigma_i^{-2}\mathbf{\A}_{j,i} + \sum_{\mathclap{\ell\neq j,i}}\sigma_\ell^{-2}\mathbf{\A}_{j,\ell}\mathbf{\A}_{i,\ell}= 0,
\end{equation}
only if $\mathbf{\A}_{i,j}=\mathbf{\A}_{j,i}= 0$ and $\mathbf{\A}_{j,\ell}\mathbf{\A}_{i,\ell}= 0$ for all $\ell\neq j,i$.
\end{assumption}
As we will show in the proof, the LHS of \cref{eq:assumption_inv_cov} is equal to $\mathbf{\Theta}_{j,i}$. It is worth noting that if the nonzero coefficients of $\mathbf{\A}$ are randomly drawn from a distribution that is absolutely continuous with respect to Lebesgue measure, then the above assumption is only violated for a set of matrices $\mathbf{\A}$ with zero Lebesgue measure. We then have the following proposition, with a proof given in Appendix \ref{proof:proposition_markov_blanket}. Note that the proposition and its proof are built upon \citet[Theorem~2]{loh2014high} in the acyclic case, which we generalize to the cyclic case.
\begin{restatable}{proposition}{PROPMARKOVBLANKET}\label{proposition:markov_blanket}
Suppose $\mathbf{X}$ follows the linear SEM in \cref{eq:linear_sem_adj} with directed cyclic graph $\G$ and inverse covariance matrix $\mathbf{\Theta}$. Under Assumption \ref{assumption:inv_cov}, the structure defined by the support of $\mathbf{\Theta}$ is the same as the moral graph of $\G$.
\end{restatable}

Asymptotically speaking, the true inverse covariance $\mathbf{\Theta}$ can be estimated by computing the inverse of empirical covariance matrix. For finite samples, \citet{ravikumar2011high} established high dimensional guarantee for estimating the support of $\mathbf{\Theta}$ using graphical Lasso \citep{friedman2008sparse}. An alternative approach \citep{meinshausen2006high} is to perform nodewise regression with Lasso \citep{robert1996lasso}, which we adopt in this work. That is, we regress the target $T$ on the other variables $[d]\setminus\{T\}$ with Lasso, from which the nonzero coefficients determine the MB of $T$.

\section{Illustrative Examples}\label{sec:examples}

\begin{figure}[H]
\centering
\includegraphics[width=0.45\textwidth]{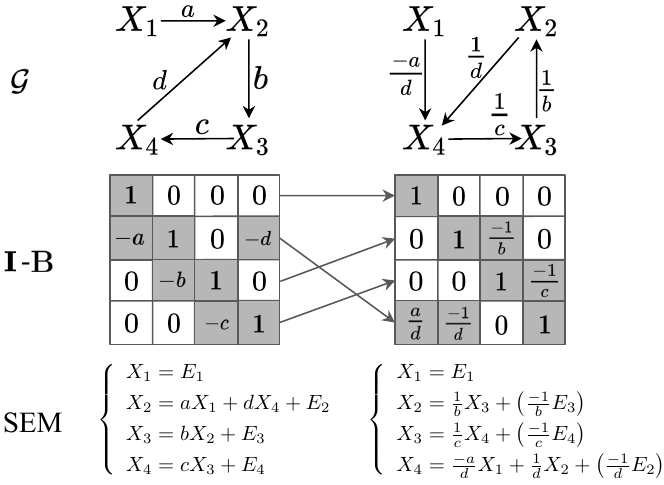}
\caption{\small Example of two equivalent cyclic LiNG models.\looseness=-1}
\label{fig:cyclic_intro_examples}
\end{figure}

This is an illustrative example of the global LiNG equivalence class $\mathcal{B}$ defined in~\cref{def:LiNG_equiv_class} of \cref{subsec:global_cyclic_ling}.
\section{POST-PROCESSING FOR LOCAL ISA-LING}\label{sec:acyclic_postprocessing}
In this section, we provide an alternative post-processing procedure to obtain the estimated structures and edge weights from the ISA solution $\W$, described in Algorithm \ref{alg:post_processing_isa}. This procedure assumes that none of $T$ and $\ch_{\G}(T)$ is part of any cycles in $\G$ (including the case where $\G$ is acyclic). The overall idea is similar to that of the regression-based approach described in Algorithm \ref{alg:inverse_direct_lingam}. That is, Algorithm \ref{alg:post_processing_isa} iteratively finds the ``sink'' node from the ISA solution that is not an ancestor of the other nodes in the remaining vertex set.
\begin{algorithm}[!h] %
	\caption{Alternative post-processing procedure of ISA solution}
	\label{alg:post_processing_isa}
	\hspace*{0.02in} {\bf Input:}
	A target vertex $T\in\mathcal{V}$, its oracle MB $\mb_\G(T)$, and ISA solution $\W$\\
	\hspace*{0.02in} {\bf Output:}
	A set of directed edges with weights

 \begin{algorithmic}[1]
	\STATE Initialize the remaining vertex set $\U_1,\U_2\coloneqq \{T\}\cup \mb_\G(T)$, and the output edge set $\K \coloneqq \emptyset$;
	\WHILE{$\U_1 \neq \emptyset$}
      \STATE \textbf{Assert} $\exists j\in \U_1$ s.t. $\|\W_{\U_2,j}\|_0=1$;
      \STATE Set $j$ as any one found in line 7;
      \STATE Let $k\in\U_2$ be s.t. $\W_{k,j}=1$;
      \IF{$\W_{k,T} \neq 0$}
        \STATE Add to $\K$ an edge $i\rightarrow j$ with weight $\W_{k,i}$ for each $i\in\U_1\backslash\{j\}$ with $\W_{k,i} \neq 0$;
      \ENDIF
      \STATE \textbf{break} if $j=T$; Otherwise set $\U_1\coloneqq \U_1\backslash\{j\}$ and $\U_2\coloneqq \U_2\backslash\{k\}$;
    \ENDWHILE
	\STATE {\bfseries Return} $\K$;
	\end{algorithmic}
\end{algorithm}

\section{PROOFS OF MAIN RESULTS}\label{sec:proofs}

\subsection{Proof of~\cref{thm:BSS_inv_is_ISA}}\label{proof:thm_BSS_inv_is_ISA}
\THMONECHARACTERIZATIONOFISAINLING*

\newpage
\begin{proof}
    For convenience denote $\W\coloneqq\B_{\Sb,\Sb}^{-1}$. Write the ISA demixed subspaces as exogenous noise combinations:\looseness=-1\vspace{2em}
\begin{align*}
\W\X_\Sb &= 
    \begin{bNiceArray}{w{c}{1cm}}[margin] \Block{3-1}{\W} \\ \\ \\ \end{bNiceArray} \ \ \cdot \ \
    \begin{bNiceArray}{w{c}{3.4cm}}[first-col,margin] \Block{3-1}{\Sb\hspace*{3mm}} & \Block{3-1}{\B_{\Sb,:}}\\ \\ \\ 
    \CodeAfter
    \OverBrace[yshift=2mm]{1-1}{1-last}{\mathcal{V}\coloneqq[d]}
    \SubMatrix{\{}{1-1}{3-1}{.}[left-xshift=1.7cm]
    \end{bNiceArray}\cdot\E \\[3em]
&=\begin{bNiceArray}{w{c}{1cm}}[margin] \Block{3-1}{\B_{\Sb,\Sb}^{-1}} \\ \\ \\ \end{bNiceArray} \ \ \cdot \ \
\begin{bNiceArray}{w{c}{1cm}|w{c}{2.2cm}}[first-col,margin]\Block{3-1}{\Sb\hspace*{2mm}} & \Block{3-1}{\B_{\Sb,\Sb}} & \Block{3-1}{\B_{\Sb,\bar{\Sb}}} \\ \\ \\ 
\CodeAfter
    \OverBrace[yshift=2mm]{1-1}{1-1}{\Sb}
    \OverBrace[yshift=2mm]{1-2}{1-last}{\bar{\Sb}\coloneqq\mathcal{V}\backslash\Sb\vspace{-0.25em}}
    \SubMatrix{\{}{1-1}{3-1}{.}[left-xshift=0.4cm]
    \end{bNiceArray}\cdot\E \\[3em]
&=\begin{bNiceArray}{w{c}{1cm}|w{c}{2.2cm}}[first-col,margin]\Block{3-1}{\Sb\hspace*{3mm}} & \Block{3-1}{\I} & \Block{3-1}{\cdots} \\ \\ \\ 
\CodeAfter
    \OverBrace[yshift=2mm]{1-1}{1-1}{\Sb}
    \OverBrace[yshift=2mm]{1-2}{1-last}{\bar{\Sb}\coloneqq\mathcal{V}\backslash\Sb\vspace{-0.25em}}
    \SubMatrix{\{}{1-1}{3-1}{.}[left-xshift=0.7cm]
    \end{bNiceArray}\cdot\E
\end{align*}
where $\E=(E_1^\intercal,\cdots,E_d^\intercal)^\intercal$ are the mutually independent exogenous non-Gaussian noise components.

To show that $\W$ is an ISA of $\X_\Sb$, we want to show that for any subspace $\Z_i \in  (\Z^\intercal_1, \dots, \Z^\intercal_k)^\intercal  =  \W\X_\Sb$ with $m\coloneqq|\Z_i|\geq 2$ (otherwise it's already a single component; $|\cdot|$ denotes dimension or cardinality), $\Z_i$ is irreducible (\cref{def:irreducibility}), i.e., there exists no invertible matrix $\Hb \in Gl(m)$ s.t. $\Hb \Z_i$ produces two or more independent subspaces (random vectors). For convenience, we denote the row indices corresponding to the row-submatrix of $\W$ that produces the subspace $\Z_i$ as $\M$ ($\M\subset \Sb$), i.e., $\Z_i = \W_{\M,:}\X_\Sb$. Similarly, rewrite it to noise combinations:\vspace{2em}
\begin{align}
\W\X_\Sb &= \ 
    \begin{bNiceArray}{w{c}{1cm}}[first-col,margin] \Block{2-1}{\M\hspace*{1mm}} & \Block{2-1}{\W_{\M,:}}\\ \\ 
    \CodeAfter
    \OverBrace[yshift=2mm]{1-1}{1-last}{\Sb}
    \SubMatrix{\{}{1-1}{2-1}{.}[left-xshift=0.3cm]
    \end{bNiceArray} \ \ \cdot \ \
    \begin{bNiceArray}{w{c}{3.4cm}}[first-col,margin] \Block{3-1}{\Sb\hspace*{3mm}} & \Block{3-1}{\B_{\Sb,:}}\\ \\ \\
    \CodeAfter
    \OverBrace[yshift=2mm]{1-1}{1-last}{\mathcal{V}\coloneqq[d]}
    \SubMatrix{\{}{1-1}{last-1}{.}[left-xshift=1.7cm]
    \end{bNiceArray}\cdot\E  \nonumber \\[3em]
&= \ \begin{bNiceArray}{w{c}{0.67cm}|w{c}{2.73cm}}[first-col,margin]\Block{2-1}{\M\hspace*{1mm}} & \Block{2-1}{\I} & \Block{2-1}{\cdots} \\ \\ 
\CodeAfter
    \OverBrace[yshift=2mm]{1-1}{1-1}{\M}
    \OverBrace[yshift=2mm]{1-2}{1-last}{\bar{\M}\coloneqq\mathcal{V}\backslash\M\vspace{-0.25em}}
    \SubMatrix{\{}{1-1}{last-1}{.}[left-xshift=0.5cm]
    \end{bNiceArray}\cdot\E \label{proof_eq:eq9}\\[2em]
    & \eqqcolon \C_\M \cdot \E,\nonumber
\end{align}
where we denote the $m\times d$ rectangle mixing submatrix in~\cref{proof_eq:eq9} as $\C_\M$. We do not use letter $\B$ for distinguishment, as it is multiplied by $\W$, and is different from submatrix from the original $\B$ mixing matrix.

Suppose for contradiction that $\Z_i$ is irreducible, i.e., there exists an invertible matrix $\Hb \in Gl(m)$ s.t. $\Hb \Z_i$ produces at least two independent subspaces, then, there must exist a partition $\Pb_1, \Pb_2$ of $[d]$ s.t.,
\begin{align}
\Hb\Z_i &= \ \Hb \cdot  \C_\M \cdot \E \nonumber\\[2em] 
&=
    \begin{bNiceArray}{w{c}{0.67cm}}[margin]
    \Block{1-1}{\Hb_1}\\ 
    \hline
    \Block{1-1}{\Hb_2} \\ 
    \end{bNiceArray} \ \ \cdot \ \
    \begin{bNiceArray}{w{c}{1.4cm}|w{c}{2cm}}[first-col,margin]\Block{2-1}{\M\hspace*{1mm}} & \Block{2-1}{\C_{\M,\Pb_1}} & \Block{2-1}{\C_{\M,\Pb_2}} \\ \\ 
\CodeAfter
    \OverBrace[yshift=2mm]{1-1}{1-1}{\Pb_1}
    \OverBrace[yshift=2mm]{1-2}{1-2}{\Pb_2=\M\backslash\Pb_1}
    \SubMatrix{\{}{1-1}{last-1}{.}[left-xshift=0.4cm]
    \end{bNiceArray}\cdot\E \label{proof_eq:eq12}\\[3em]
    &= \begin{bNiceArray}{w{c}{1.4cm}|w{c}{2cm}}[margin] \mathbf{0} & \cdots \\ 
    \hline \cdots & \mathbf{0}\\ 
\CodeAfter
    \OverBrace[yshift=2mm]{1-1}{1-1}{\Pb_1}
    \OverBrace[yshift=2mm]{1-2}{1-2}{\Pb_2=\M\backslash\Pb_1}
    \end{bNiceArray},\label{proof_eq:eq13}
\end{align}
i.e., $\Hb_1 \Z_i$ and $\Hb_2 \Z_i$ are linear combinations of disjoint sets of exogenous noise components, and thus by the Darmois-Skitovitch theorem~\citep{darmois1953analyse,skitovitch1953property}, they are mutually independent.

By~\cref{proof_eq:eq12,proof_eq:eq13} we have that row vectors of $\Hb_1$ lie in $\operatorname{nullspace}(\C_{\M,\Pb_1}^\intercal)$, and row vectors of $\Hb_2$ lie in $\operatorname{nullspace}(\C_{\M,\Pb_2}^\intercal)$. Also, since $\C_{\M,\M}=\I$, $\operatorname{rank}(\C_\M) = m$ (i.e., full row rank), so,
\begin{align*}
    \operatorname{rank}(\C_{\M,\Pb_1}) + \operatorname{rank}(\C_{\M,\Pb_2}) \geq \operatorname{rank}(\C_{\M,\Pb_1} | \C_{\M,\Pb_2}) = m,
\end{align*}
and thus
\begin{align*}
    & m - \operatorname{nullity}(\C_{\M,\Pb_1}^\intercal) + m - \operatorname{nullity}(\C_{\M,\Pb_2}^\intercal) \geq = m, \\
     \text{i.e., } &\operatorname{nullity}(\C_{\M,\Pb_1}^\intercal) + \operatorname{nullity}(\C_{\M,\Pb_2}^\intercal) \leq m
\end{align*}

Consider the following two cases:

\begin{enumerate}%
    \item[$1^\circ$] When $\operatorname{nullity}(\C_{\M,\Pb_1}^\intercal) + \operatorname{nullity}(\C_{\M,\Pb_2}^\intercal) < m$, even when these two nullspaces are linearly independent, the number of their supports is less than $m$ and there are not enough number of linearly independent row vectors to fill into $\Hb_1$ and $\Hb_2$ to form an invertible $\Hb$. Contradicted with our hypothesis.
    \item[$2^\circ$] When $\operatorname{nullity}(\C_{\M,\Pb_1}^\intercal) + \operatorname{nullity}(\C_{\M,\Pb_2}^\intercal) = m$, the above independence condition $\Hb_1\Z_i \indep \Hb_2\Z_i$ is nontrivial (i.e., both are still random vectors with covariance, instead of a collapsing constant zero) only when:
    \begin{align*}
\left\{ \begin{array}{l}
\operatorname{nullity}(\C_{\M,\Pb_1}^\intercal)>0 \\
\operatorname{nullity}(\C_{\M,\Pb_2}^\intercal)>0
\end{array}\right.
\end{align*}
    However, this is impossible:

    Suppose for contradiction that $0 < \operatorname{rank}(\C_{\M,\Pb_1}), \operatorname{rank}(\C_{\M,\Pb_2}) < m$, then at least the $\C_{\M,\M}=\I$ part must be separated, i.e., $\biggl\{\begin{array}{l}
\M\not\subset\Pb_1 \\
\M\not\subset\Pb_2
\end{array}$. Then, there must be a partition of $\M$ into into smaller respective subsets $(\M_u,\M_v)$ (we do not use $\M_1,\M_2$ to distinguish from the row indices for $\Hb_1,\Hb_2$) s.t. $\biggl\{\begin{array}{l}
\M_u\subset\Pb_1 \\
\M_v\subset\Pb_2
\end{array}$, then, since $\operatorname{rank}(\C_{\M,\Pb_1}) = |\M_u|$ and $\C_{\M_u,\M_u}=\I$, $\C_{\M_v,\M_u}$ must be all zeros. Further, since the $\M_v$ rows are linear combinations of the $\M_u$ rows, $\C_{\M_v,\Pb_1 \backslash \M_u}$ must also be all zeros. Same applies to $\C_{\M,\Pb_2}$. We have:
\vspace{2em}
\begin{equation*}
    \C_{\M,\Pb_1} = \begin{bNiceArray}{w{c}{0.33cm}|w{c}{1.07cm}}[first-col,margin]{\M_u\hspace*{1mm}} & {\I} & {\cdots} \\
    \hline
    {\M_v\hspace*{1mm}} & {\mathbf{0}} & {\mathbf{0}}\\ 
\CodeAfter
    \OverBrace[yshift=2mm]{1-1}{1-1}{\M_u}
    \OverBrace[yshift=2mm]{1-2}{1-2}{\Pb_1\backslash\M_u}
    \SubMatrix{\{}{1-1}{1-1}{.}[left-xshift=0.35cm]
    \SubMatrix{\{}{2-1}{2-1}{.}[left-xshift=0.35cm]
    \end{bNiceArray} \text{\quad and\quad}
    \C_{\M,\Pb_2} = \begin{bNiceArray}{w{c}{0.33cm}|w{c}{1.67cm}}[first-col,margin]{\M_u\hspace*{1mm}} & {\mathbf{0}} & {\mathbf{0}} \\
    \hline
    {\M_v\hspace*{1mm}} & {\I} & {\cdots} \\ 
\CodeAfter
    \OverBrace[yshift=2mm]{1-1}{1-1}{\M_v}
    \OverBrace[yshift=2mm]{1-2}{1-2}{\Pb_2\backslash\M_v}
    \SubMatrix{\{}{1-1}{1-1}{.}[left-xshift=0.35cm]
    \SubMatrix{\{}{2-1}{2-1}{.}[left-xshift=0.35cm]
    \end{bNiceArray},
\end{equation*}
However, in this case, $\C_{\M_u,:}\E \indep \C_{\M_v,:}\E$, as they share disjoint non-Gaussian $\E$ components. This contradicts with the initial hypothesis on a nontrivial subspace $\Z_i$, as $\M_u$ and $\M_v$ in $\Sb$ will not be mixed in $\M$, but rather produce two independence subspaces at the very beginning.
\end{enumerate}

From the above contradiction, every $\Z_i$ must be irreducible. So $\B_{\Sb,\Sb}^{-1}$ is an ISA.
\end{proof}

Note that while we are not the first to use ISA in linear non-Gaussian models with latent variables, this work is, to the best of our knowledge, the first with an identifiability guarantee. As shown in~\cref{example:hidden_confounder_ica,example:BSS_inv_vs_ASS,example:BSS_inv_no_diagonal_ones,example:permutation_diagonal_nonzeros_on_isa}, the characterization and post-processing of ISA are highly nontrivial. Some prior works~\citep{sanchez2019estimating} simply treated ISA solutions like ICA solutions and applied the same post-processing, resulting in inaccuracies. Other works used ISA mainly for downstream steps e.g., OICA~\citep{hoyer2008estimation} or independence tests~\citep{dai2022independence}, but not directly for the LiNG model identification. We believe that the generalized characterization of ISA solutions provided in~\cref{thm:BSS_inv_is_ISA} can be helpful for future works on causal discovery with latent variables.

\subsection{Proof of~\cref{lemma:when_parents_are_in_ISA}}\label{proof:lemma_when_parents_are_in_ISA}
\LEMMAWHENPARENTSAREINISA*
\begin{proof}
For any vertex $i$ and vertex set $\Sb$ of $\mathcal{V}$ with $i\in\Sb$ and $\pa_\G(i) \subset \Sb$, we can write the variable $X_i$ as
    \begin{align}
        X_i &= \B_i \E  \label{proof_eq:lem1eq1} \\
            &= \A_{i,\pa_\G(i)}\X_{\pa_\G(i)} + E_i \label{proof_eq:lem1eq12}\\
            &= \A_{i,\Sb}\X_{\Sb} + E_i \label{proof_eq:lem1eq13} \\
            &= \A_{i,\Sb}\B_{\Sb}\E + E_i \label{proof_eq:lem1eq2}
    \end{align}    
    where subscripts of index of indices denote the corresponding row/column submatrices. \cref{proof_eq:lem1eq12} to~\cref{proof_eq:lem1eq13} is trivial because $i$ has no parents from outside of $\Sb$, i.e., $\A_{i,\Sb\backslash\pa_\G(i)} = \mathbf{0}$.

    By \cref{proof_eq:lem1eq1}=\cref{proof_eq:lem1eq2}, we have
    \begin{align}
        \B_i = \A_{i,\Sb}\B_\Sb + \mathbbm{1}_{i}^{|\Sb|}, \label{proof_eq:lem1eq3}
    \end{align}
    where $\mathbbm{1}_{i}^{|\Sb|}$ denotes the row vector of dimension $|\Sb|$ with only the $i$-th indexed entry being one, and elsewhere zeros. \cref{proof_eq:lem1eq3} tells that all noise components (``ancestors'') coming into $X_i$, except for the $E_i$ itself, must go through $\pa_G(i)$. 

    Keep only the columns of $\Sb$ on \cref{proof_eq:lem1eq3}, we have
    \begin{align}
        \B_{i,\Sb} = \A_{i,\Sb}\B_{\Sb,\Sb} + \mathbbm{1}_{i}^{|\Sb|}, \label{proof_eq:lem1eq4}
    \end{align}

    By \cref{proof_eq:lem1eq4} we have
    \begin{align}
        (\B_{i,\Sb} - \mathbbm{1}_{i}^{|\Sb|}) \B_{\Sb,\Sb}^{-1} = \A_{i,\Sb}, \label{proof_eq:lem1eq5}
    \end{align}

    where note that $\B_{\Sb,\Sb}^{-1}$ is exactly the ISA characterization (\cref{subsec:BMM_inv}) for $\X_\Sb$. Expand \cref{proof_eq:lem1eq5} we have
    \begin{align}
        & \B_{i,\Sb}\B_{\Sb,\Sb}^{-1} - \mathbbm{1}_{i}^{|\Sb|}\B_{\Sb,\Sb}^{-1} = \A_{i,\Sb},\text{ i.e.,} \nonumber\\
        & \mathbbm{1}_{i}^{|\Sb|} - (\B_{\Sb,\Sb}^{-1})_{i} = \A_{i,\Sb},\label{proof_eq:lem1eq6}
    \end{align}
    \cref{proof_eq:lem1eq6} tells that the $i$-th row of the ISA characterization $\B_{\Sb,\Sb}^{-1}$ is exactly $\mathbbm{1}_{i}^{|\Sb|} - \A_{i,\Sb}$, i.e., the $i$-th row of $\I - \A_{\Sb,\Sb}$. In other words, $(\B_{\Sb,\Sb}^{-1})_{i} \B_\Sb = \mathbbm{1}_{i}^{|\Sb|}$. Then, substitute \cref{proof_eq:lem1eq6} into the demixed subspaces,
    \begin{align*}
        (\B_{\Sb,\Sb}^{-1})_{i} \X_\Sb = X_i - \A_{i,\Sb}\X_\Sb = E_i,
    \end{align*}
    i.e., the indpendent component (1-dim subspace) of $E_i$ is exactly recovered.

    Finally, with the subspace-wise permutation and scaling indeterminacies of ISA (\cref{thm:indeterminacies_of_isa}), there must be a row in any ISA solution $\W$ being proportional to $(\I-\B_{\Sb,\Sb})_i$, and the decomposed component also.
\end{proof}

\subsection{Proof of~\cref{lemma:T_column_as_T_and_children}}\label{proof:lemma_T_column_as_T_and_children}
\LEMMATCOLUMNASTANDCHILDREN*
The proof is apparent and is almost the same as the above for~\cref{lemma:when_parents_are_in_ISA}: since all of $T$'s children is included (as the ``all parents'' in that of~\cref{lemma:when_parents_are_in_ISA}), all the weights outgoing from $T$ can also be correctly estimated. This can also be seen from the expression of $(\B_{\Sb,\Sb}^{-1})_{i,T}$ entries in~\cref{eq:BSS_inv_outside}, that ISA has indeterminacies of subspace-wise permutations and scalings, and that rows permutations of $(\B_{\Sb,\Sb}^{-1})_{i,T}$ with nonzero diagonal entries directly correspond to each of that on $\B_{\Sb,\Sb}^{-1}$ (the equivalence class $\mathcal{B}$ in~\cref{def:LiNG_equiv_class}).

\subsection{Proof of~\cref{thm:correct_isa_ling}}\label{proof:thm_correct_isa_ling}
\THMCORRECTNESSOFLOCALISALING*

To show the correctness of~\cref{alg:local_isa_ling}, we first show the correctness of the ``admissible'' block permutations defined in~\cref{def:admissible_block_permutations}. To put it formally, we have the following lemma:
\begin{lemma}\label{lemma:groupwise_permutation_equivalent}
    Let $\C$ be an arbitrary $m\times m$ invertible matrix, $\PRT$ be an arbitrary partition of $[m]$.
    
    Denote by $\Pi_\C$ as the set of all the rows permutations of $\C$ that result in nonzero diagonal entries, i.e.,
    $$\Pi_\C \coloneqq \{\pi : \Pb_\pi\C \text{ has all the nonzero diagonal entries.}\},$$
    Denote by $\Pi_{\C;\PRT}$ as all the rows permutations that result in invertible diagonal blocks on general scaled $\C$, i.e.,
    $$\Pi_{\C;\PRT} \coloneqq \{\pi : \exists \D_\PRT, \forall \Sb \in \PRT, \operatorname{ rank}((\Pb_\pi\D_\PRT\C)_{\pi[\Sb],\pi[\Sb]})=|\Sb|\},$$
    where $\D_\PRT$ is any general scaling matrix (defined in~\cref{subsec:isa_definition}) consistent with $\PRT$.
    
    For any two permutations $\pi$ and $\tau$ of $[m]$, we say they are groupwise equivalent regarding a partition $\PRT$ of $[m]$, denoted by $\pi\sim_\PRT\tau$, if and only if $\forall \Sb \in \PRT,$ $\pi[\Sb]$ and $\tau[\Sb]$ have exactly the same elements. Then, $\Pi_\C$ and $\Pi_{\C;\PRT}$ are equivalent up to groupwise permutations, i.e.,

    \begin{enumerate}
        \item $\forall\tau\in \Pi_{\C;\PRT}, \ \exists \pi\in\Pi_\C, \ \text{s.t. } \pi\sim_\PRT\tau$;
        \item $\forall\pi\in \Pi_{\C},$ if $\forall \Sb\in\PRT$, $(\Pb_\pi\C)_{\pi[\Sb],\pi[\Sb]}$ is invertible, then $\exists \tau\in\Pi_{\C;\PRT}, \ \text{s.t. } \pi\sim_\PRT\tau$.
    \end{enumerate}
\end{lemma}

\cref{lemma:groupwise_permutation_equivalent} tells that all permutations that can result in nonzero diagonal entries on an invertible matrix are groupwise equivalent to all permutations that in result in invertible diagonal blocks on the same matrix corresponding to a given partition of the row indices. Note that $1.$ is universally true, while $2.$ needs an additional mild assumption that the partition and nonzero-diagonal permutation itself result in invertible diagonal blocks. Consider a counterexample: $\C = \begin{bNiceArray}{ccc}
   1 & 1  & 2 \\
   1  & 1 & 3 \\
   1&0&1
\end{bNiceArray}$ is invertible, and a partition of row indices $\PRT = \{(1,2), (3,)\}$. Clearly the identity $\pi$ (i.e., $\Pb_\pi = \I$) is in $\Pi_\C$, with $\C$ already having nonzero diagonal entries. However, we cannot find any $\tau\in\Pi_{\C;\PRT}$ with $\pi\sim_\PRT\tau$:
\[\D_\PRT\C = \begin{bNiceArray}{ccc}
   a & b  & 0 \\
   c  & d & 0 \\
   0&0&e
\end{bNiceArray}\begin{bNiceArray}{ccc}
   1 & 1  & 2 \\
   1  & 1 & 3 \\
   1&0&1
\end{bNiceArray}=\begin{bNiceArray}{ccc}
   a+b & a+b  & 2a+3b \\
   c+d  & c+d & 2c+3d \\
   e&0&e
\end{bNiceArray},\]
either $\tau = (1,2,3)$ or $(2,1,3)$ is not in $\Pi_{\C;\PRT}$, because $\begin{bNiceArray}{cc}
   a+b & a+b \\
   c+d  & c+d
\end{bNiceArray}$ is already not invertible itself. Therefore, we make an additional assumption for $2.$, which is, as we can see later, trivially satisfied for our choice of invertible $\C$.\looseness=-1

Now we prove the correctness of~\cref{lemma:groupwise_permutation_equivalent}:
\begin{proof}
    First, $\Pi_\C$ is nonempty, because $\C$ is invertible, $\operatorname{det}(\C) = \sum_\pi \operatorname{sgn}(\pi)\Pi_{i=1}^m{\C_{i,\pi_i}} \neq 0$, then at least there is one $\pi$ s.t. $\forall i=1,\cdots,m$, $\C_{i,\pi_i}\neq 0$, and so the inverse of this $\pi$ will do a rows permutation with nonzero diagonals. The nonemptiness of $\Pi_{\C;\PRT}$ can be proved using a similar idea (i.e., $\forall \PRT, \forall \C, \exists \tau$ s.t. $\Pb_\tau\D_\PRT\C$ has invertible diagonal blocks), but using group decomposition of the determinant expression, called ``Generalized Laplacian expansion''~\citep{janjic2008proof}.

    To prove $1.$, for any $\tau\in\Pi_{\C;\PRT}$, initialize a new empty $\pi$. For any group $\Sb\in\PRT$, by definition, the block $(\Pb_\tau\D_\PRT\C)_{\tau[\Sb],\tau[\Sb]}$ is invertible. Note that $(\Pb_\tau\D_\PRT\C)_{\tau[\Sb],\tau[\Sb]} = (\Pb_\tau)_{\tau[\Sb],\Sb} \cdot (\D_\PRT)_{\Sb,\Sb} \cdot \C_{\Sb,\tau[\Sb]}$, so $\C_{\Sb,\tau[\Sb]}$ must also be invertible. By above nonemptiness, we know that  $\C_{\Sb,\tau[\Sb]}$ can be row permutated to one with nonzero diagonal entries. Then, we set the corresponding indices as this permutation, i.e., set $(\Pb_\pi)_{\tau[\Sb],\Sb}$ as this permutation submatrix. Then $\Pb_\pi\C$ has nonzero diagonals, and for any $\Sb$, $\pi[\Sb]$ and $\tau[\Sb]$ have the same elements (row indices).\looseness=-1

    To prove $2.$, it sufficies to show that for any $\pi\in\Pi_\C$, there is also $\pi\in\Pi_{\C;\PRT}$. For any $\Sb\in\PRT$, consider the principal submatrix $(\Pb_\pi\C)_{\pi[\Sb],\pi[\Sb]}$, which is assumed to be invertible. Note that $(\Pb_\pi\C)_{\pi[\Sb],\pi[\Sb]} = (\Pb_\pi)_{\pi[\Sb],\Sb}\cdot\C_{\Sb,\pi[\Sb]}$, and so $\C_{\Sb,\pi[\Sb]}$ is invertible. For any $\D_\PRT$, we have $(\Pb_\pi\D_\PRT\C)_{\pi[\Sb],\pi[\Sb]} = (\Pb_\pi)_{\pi[\Sb],\Sb} \cdot (\D_\PRT)_{\Sb,\Sb} \cdot \C_{\Sb,\pi[\Sb]}$, where each factor is invertible, so $(\Pb_\pi\D_\PRT\C)_{\pi[\Sb],\pi[\Sb]}$ is also invertible. Then, $\pi\sim_\PRT\pi$ trivially.
\end{proof}

By setting the invertible matrix $\C$ as the ISA characterization $\B_{\Sb,\Sb}^{-1}$ (\cref{subsec:BMM_inv}), we know that the admissible post-processing of rows permutation on any general ISA solution matrices $\W$ will make all subspaces at the correct location. The additional assumption for 2. that diagonal blocks of $\C$ are invertible echoes the `weak stability' assumption mentioned in~\citep{hyttinen2012learning}. For correctness, specifically, the $1$-dim subspaces (independent components), including the $T$ and $T$'s children that we are interested in, can be identified at the correct location (for each LiNG model in the equivalence class $\mathcal{B}$). The last step left is to associate the $\B_{\Sb,\Sb}^{-1}$ to the local adjacencies $\A_{\Sb,\Sb}$. By~\cref{eq:BSS_inv_outside}, though in general $\B_{\Sb,\Sb}^{-1}\neq \A_{\Sb,\Sb}$, they must be equal on the rows of $T$ and $T$'s children, as their parents are all included in $\Sb$. Finally, the correctness of~\cref{alg:local_isa_ling} is proved.

\subsection{Proof of~\cref{coro:local_stable}}\label{proof:coro_local_stable}
\COROLOCALSTABLE*
\begin{proof}
By the indeterminacy of ISA (\cref{thm:indeterminacies_of_isa}), the LiNG's ISA characterization (\cref{thm:BSS_inv_is_ISA}), and the characterization of the LiNG global equivalence class (\cref{def:LiNG_equiv_class}), we know that for each admissible $\A'$ at the line 8 of~\cref{alg:local_isa_ling}, there exists a ground-truth model $\A \in \mathcal{B}$ with the corresponding permutation, s.t. $\A' = \I - \D (((\I - \A)^{-1})_{\Sb,\Sb})^{-1}$, where $\Sb = \{T\}\cup\mb_\G(T)$, and $\D$ is the general scaling matrix for diagonal ones.

When the cycles in $\G$ are disjoint, all graphs in the LiNG equivalence class have disjoint cycles, and the unique global stable model $\A^*$ is the one where all cycles' products have absolute values less than one. As here stability is determined merely by the cycle products, to prove~\cref{coro:local_stable}, we only need to show the following statement:

Consider a LiNG equivalence class $\mathcal{B}$ where cycles are disjoint. For any adjacency matrix $\A\in \mathcal{B}$ and its corresponding graph $\G$ and mixing matrix $\B \coloneqq (\I - \A)^{-1}$, for any $\Sb \subset \mathcal{V}$ (not necessarily a vertex and its Markov blanket), we initialize a local graph termed $\G^{(\Sb)}$ over $\Sb$ from the local adjacency matrix termed $\A^{(\Sb)} \coloneqq \I - \B_{\Sb,\Sb}^{-1}$. From~\cref{example:BSS_inv_no_diagonal_ones} we know that $\B_{\Sb,\Sb}^{-1}$ does not necessarily have all diagonals as ones, so we further remove all self-loops on $\G^{(\Sb)}$. Then, all cycles in this local $\G^{(\Sb)}$, if any, must be disjoint. Further, for each local cycle in $\G^{(\Sb)}$ consisting of vertices $\mathbf{C}^{(\Sb)} \subset \Sb$, we have the followings:
\begin{enumerate}
    \item All vertices on the local cycle need no row scalings, i.e., $\forall i \in \mathbf{C}^{(\Sb)}, \ (\B_{\Sb,\Sb}^{-1})_{i,j} = 1$, and
    \item There exists a global cycle in $\G$ consisting of vertices $\mathbf{C}$ with $\mathbf{C}^{(\Sb)} \subset \mathbf{C}$ (in a consistent ordering), and,
    \item The cycle product of this local cycle $\mathbf{C}^{(\Sb)}$ on $\G^{(\Sb)}$ equals the cycle product of the global cycle $\mathbf{C}$ on $\G$.
\end{enumerate}

The above statement follows from~\cref{eq:BSS_inv_outside}: the $(i,j)$-th entry of $\I - \B_{\Sb,\Sb}^{-1}$ corresponds not only to the direct causal effect from $j$ to $i$, but also the total causal effect from $j$ to $i$ through all other variables outside of $\Sb$. For the above point 1., a vertex $i$ has non-unit diagonal $\B_{\Sb,\Sb}^{-1} \neq 1$ only when $i$ is involved in a cycle where all the remaining vertices in this cycle is not in $\Sb$ (see~\cref{example:BSS_inv_no_diagonal_ones}), so that this cycle appear as a self-loop on $i$ relative to $\Sb$. As cycles are disjoint, $i$ cannot belong to any cycle in $\G^{(\Sb)}$. Point 1 shows that whenever a cycle can appear locally, the edge weights on this cycle must follow exactly from $\B_{\Sb,\Sb}^{-1}$, without any scalings. Then, using the characterization in~\cref{eq:BSS_inv_outside}, points 2 and 3 show that the cycle products can be preserved locally. A local stable model (with disjoint cycles and abs(cycle products)$<1$) must correspond to a global model with those stable cycles, which, in~\cref{alg:local_isa_ling}'s case, implies the correct local stable model among $\{T\} \cup \ch_{\G^*}(T)$.
\end{proof}

\begin{wrapfigure}[10]{r}{0.18\textwidth}
\begin{tikzpicture}[inner sep=0.5pt, >=stealth]
    \node (X1a) {$X_1$};
    \node (X2a) [right=15pt of X1a] {$X_2$};
    \node (X3a) [right=15pt of X2a] {$X_3$};
    \path [->] (X1a) edge[bend left=30] node[above=2pt]{\scriptsize{$a$}} (X2a);
    \path [->] (X2a) edge[bend left=30] node[below=1pt]{\scriptsize{$b$}} (X1a);
    \path [->] (X2a) edge[bend left=30] node[above=2pt]{\scriptsize{$c$}} (X3a);
    \path [->] (X3a) edge[bend left=30] node[below=1pt]{\scriptsize{$d$}} (X2a);
    \node [right=-1pt of X3a, yshift=-10pt] {\small{\textit{(i)}}};

    \node (X1c) [below=25pt of X1a] {$X_1$};
    \node (X2c) [right=15pt of X1c] {$X_2$};
    \node (X3c) [right=15pt of X2c] {$X_3$};
    \path [->] (X3c) edge[bend right=30] node[above=1pt]{\scriptsize{$1/c$}} (X2c);
    \path [->] (X2c) edge[bend right=30] node[below=1pt]{\scriptsize{$1/d$}} (X3c);
    \path [->] (X2c) edge[bend left=30] node[above=2pt]{\scriptsize{$b$}} (X1c);
    \path [->] (X1c) edge[bend right=70] node[pos=0.3,left=3pt]{\scriptsize{$-a/d$}} (X3c);
    \node [right=-1pt of X3c, yshift=-10pt] {\small{\textit{(ii)}}};

    \node (X1b) [below=30pt of X1c] {$X_1$};
    \node (X2b) [right=15pt of X1b] {$X_2$};
    \node (X3b) [right=15pt of X2b] {$X_3$};
    \path [->] (X2b) edge[bend right=30] node[above=1pt]{\scriptsize{$1/a$}} (X1b);
    \path [->] (X1b) edge[bend right=30] node[below=1pt]{\scriptsize{$1/b$}} (X2b);
    \path [->] (X2b) edge[bend left=30] node[above=2pt]{\scriptsize{$c$}} (X3b);
    \path [->] (X3b) edge[bend left=70] node[pos=0.3,right=2pt]{\scriptsize{$-d/a$}} (X1b);
    \node [right=-1pt of X3b, yshift=-10pt] {\small{\textit{(iii)}}};
    
\end{tikzpicture}
\label{fig:example_3_nodes_2_cycles}
\end{wrapfigure}
\textbf{Remark.} \ \  
However, note that the above proof relies on the assumption that the cycles in $\G$ are disjoint. However, when some cycles in $\mathcal{G}$ intersect, things become more complex. The figure to the right shows an example of three cyclic graphs in a LiNG equivalence class with intersected cycles. Globally, there may be none (e.g., when $a=b=c=d=0.8$) or multiple (e.g., both \textit{(i)} and \textit{(ii)} when $a=b=c=0.8$, $d=-2$) global stable models. In this case, while our method can still identify local correspondings of all equivalent models (\cref{thm:correct_isa_ling}), and some unstable solutions may be partially eliminated (using the local stability constraint), we show that the exact identification of the global stable solutions from local variables alone becomes inherently impossible, because intuitively, external cycles appear as self-loops on the local variables. Note that when cycles intersect, the simple cycles' products cannot be related to the stability directly anymore: a LiNG model can be stable with some cycle products larger than one, and a LiNG model with all abs(cycle products) less than one can also be unstable. Even if we force to use cycle products to define stability (as some papers~\citep{rothenhausler2015backshift} do), the abovementioned unidentifiability issue of global stable solutions from local variables still remains.

\subsection{Proof of~\cref{lemma:independent_noise}}\label{proof:lemma_independent_noise}
\LEMMAINDEPENDENTNOISE*
The proof can be referred to~\citep{shimizu2011directlingam}, by using Darmois-Skitovitch theorem~\citep{darmois1953analyse,skitovitch1953property}.

\subsection{Proof of~\cref{lemma:correct_regression_coefs}}\label{proof:lemma_correct_regression_coefs}
\LEMMACORRECTREGRESSIONCOEFS*
The proof follows naturally from~\cref{lemma:when_parents_are_in_ISA} (in the acyclic graph case).

\subsection{Proof of~\cref{thm:inv_direct_lingam_correct}}\label{proof:thm_inv_direct_lingam_correct}
\THMCORRECTNESSOFINVDIRECTLINGAM*
\begin{proof}
    At every iteration of~\cref{alg:inverse_direct_lingam}, consider a `last' remaining vertex $j\in\U$ s.t. there exists no other $j'\in\U$ as $j$'s descendant on $\G$. $1^\circ$ if $j$ is $T$ or $T$'s child, since $\pa_\G(j)\subset \mb_\G(T)$, and since none of $\pa_\G(j)$ can be removed earlier, we have $\pa_\G(j)\subset \U$. With all the parents in $\U$ and no descendants in $\U$, regress $j$ on $\U\backslash\{j\}$ will produce independent residual, and the nonzero coefficients correspond to the true direct parents with true weights, i.e., $\beta_{\U\backslash\{j\}\rightarrow j}^i = \A_{j,i}$; $2^\circ$ if $j$ is $T$'s spouse, since $j$ is `last', none of $j$ and $T$'s common children and descendants are in $\U$. Also since $\pa_\G(T)\subset \U$, every confounding path between $j$ and $T$ is blocked, and thus $\beta_{\U\backslash\{j\}\rightarrow j}^T = 0$; $3^\circ$ it is impossible for $j$ to be $T$'s parents, since when $T$ pops out, the program breaks.
\end{proof}

\subsection{Proof of~\cref{proposition:markov_blanket}}\label{proof:proposition_markov_blanket}

We first state the following lemmas from \citet[Lemmas 1 \& 2]{ng2021reliable} (which were based on \citet{loh2014high}). The original lemmas in \citet{ng2021reliable} focus on linear acyclic SEM, but their proofs do not make use of the acyclicity constraint. Therefore, we restate the lemmas here for cyclic graphs, whose proofs are similar to the acyclic case and omitted.
\begin{lemma}\label{lemma:theta_entries}
Suppose $\mathbf{X}$ follows the linear SEM in \cref{eq:linear_sem_adj} with directed cyclic graph $\G$ and inverse covariance matrix $\mathbf{\Theta}$.
The entries of $\Theta$ are given by
\begin{alignat*}{3}
\mathbf{\Theta}_{j,i} &= -\sigma_j^{-2}\A_{i,j}-\sigma_i^{-2}\A_{j,i} + \sum_{\mathclap{\ell\neq j,i}}\sigma_\ell^{-2}\A_{j,\ell}\A_{i,\ell}, \qquad&& \forall j\neq k, \\
\mathbf{\Theta}_{j,j} &= \sigma_j^{-2} + \sum_{\ell\neq j}\sigma_\ell^{-2}\A_{j,\ell}^2, \qquad&& \forall j. \\
\end{alignat*}
\end{lemma}

\begin{lemma}\label{lemma:theta_subgraph}
Suppose $\mathbf{X}$ follows the linear SEM in \cref{eq:linear_sem_adj} with directed cyclic graph $\G$ and inverse covariance matrix $\mathbf{\Theta}$. Then, the structure defined by the support of $\mathbf{\Theta}$ is a subgraph of the moral graph of $\G$.
\end{lemma}

We then provide the proof for the following proposition.
\PROPMARKOVBLANKET*
\begin{proof}
By Lemma \ref{lemma:theta_subgraph}, the structure defined by the support of $\mathbf{\Theta}$ is a subgraph of the moral graph of $\G$. By Assumption \ref{assumption:inv_cov} and Lemma \ref{lemma:theta_entries}, if $\mathbf{\Theta}_{j,i}=0$, then we have $\A_{i,j}=\A_{j,i}= 0$ and $\A_{j,\ell}\A_{i,\ell}= 0$ for all $\ell\neq j,i$, which, by definition, indicates that $i$ and $j$ are not adjacent in the moral graph of $\G$. This indicates that the moral graph of $\G$ is a subgraph of the structure defined by the support of $\mathbf{\Theta}$.
\end{proof}

\section{SUPPLEMENTARY EXPERIMENTS DETAILS}\label{sec:supplementary_experiment_details}
We provide additional details for the experiments conducted in \cref{sec:experiments}. Specifically, we provide the implementation details of our method and the baselines in \cref{sec:implementation_details_ours,sec:implementation_details_baselines}, respectively. We then describe in \cref{sec:metrics} the performance metrics used in our experiments.

\subsection{Implementation Details of Local ISA-LiNG}\label{sec:implementation_details_ours}
The proposed Local ISA-LiNG method involves estimating the demixing matrix with ISA (see \cref{alg:local_isa_ling}). One could use the ISA procedure developed by \citet{theis2006general}. In this work, we use ICA to first estimate the  components, and then use independence test, i.e. the Hilbert-Schmidt independence criterion (HSIC) \citep{gretton2007kernel}, to identify the subspaces from the components estimated by ICA. Such a procedure is found to work well in practice. For the HSIC test in our experiments, we use a significance level of $0.05$. Furthermore, as suggested by \citet{lacerda2008discovering}, we adopt ICA with sparse connection \citep{zhang2006ica} for the specific ICA procedure. In the acyclic case , we use a more efficient post-processing procedure to identify the edges from the demixing matrix $\W$, described in Appendix \ref{sec:acyclic_postprocessing}.

We perform nodewise regression with Lasso \citep{meinshausen2006high} to estimate the MB of each target variable (see \cref{sec:markov_blanket} for more details). Moreover, when applying \cref{alg:local_isa_ling,alg:post_processing_isa}
to identify the edges from the estimated demixing matrix, we use a threshold of $0.05$ to set the entries with small absolute values to zero. All experiments are conducted on $2$ CPUs with $4$GB of memory. The code is available at \url{https://github.com/MarkDana/local-ling-discovery}.

\subsection{Implementation Details of Existing Methods}\label{sec:implementation_details_baselines}
We provide the implementation details of several existing methods considered in our experiments. All experiments are conducted on $2$ CPUs with $4$GB of memory. For Local A* and GSBN, we perform nodewise regression with Lasso \citep{meinshausen2006high} to estimate the MB of each target variable, similar to our Local ISA-LiNG method. For CMB and LDECC, we use their default method for MB estimation. In the following, we describe the implementation details of each method.

\textbf{ICA-LiNG.} \ \  
As suggested by \citet{lacerda2008discovering}, we use ICA with sparse connection \citep{zhang2006ica} for the specific ICA procedure, and a threshold of $0.05$ for the demixing matrix, similar to our method.

\textbf{Local A*.} \ \  
The original Local A* method \citep{ng2021reliable} applies exact search strategy like A* \citep{yuan2013learning} on the target $T$, its MB $\mb_\G(T)$, and MB of each variable in $\mb_\G(T)$. In the estimated structure, the method then identifies all (1) undirected edges involving $T$ and (2) v-structures involving $T$ to be the final local structure around $T$. In our modification, we apply A* on only $T$ and its MB $\mb_\G(T)$. In this case, the modified method identifies all (1) undirected edges involving $T$ and (2) v-structures of which $T$ is not the collider to be the final estimated local structure around $T$. The resulting algorithm performs local discovery on only $T$ and its MB $\mb_\G(T)$. Here, we use BIC score \citep{schwarz1978estimating} for the A* search.

\textbf{GSBN.} \ \  
The original GSBN method \citep{margaritis1999bayesian} applies certain rules based on conditional independence tests to identify all (1) undirected edges involving $T$ and (2) v-structures of which $T$ is the collider. The latter requires information of two-step MBs, i.e., $\mb_\G(T)$ and MB of each variable in $\mb_\G(T)$. In our modification, we adopt different but similar rules to identify (1) undirected edges involving $T$ and (2) v-structures of which $T$ is not the collider:
\begin{enumerate}
\item For each $X\in\operatorname{mb}_{\mathcal{G}}(T)$, determine $X$ to be a (direct) neighbor of $T$ if $T\not\indep X | S$ for all $S\subseteq \mb_{\mathcal{G}}(T)\setminus\{Y\}$.
\item Given the neighbors of target $T$ identified in the first step, we use the following rule to identify the v-structures: for each $Z\in\operatorname{mb}_{\mathcal{G}}(T)$ and $Y$ being a neighbor of $T$, determine $T\rightarrow Y \leftarrow Z$ to be a v-structure if $T\not\indep Z | S\cup\{Y\}$ for all $S\subseteq \mb_{\mathcal{G}}(T)\setminus\{Y,Z\}$.
\end{enumerate}
The resulting algorithm performs local discovery on only $T$ and its MB $\mb_\G(T)$. Here, we use Fisher Z test with significance level of $0.05$ for identifying conditional independence relations.

\textbf{CMB.} \ \ 
For the CMB method \citep{gao2015local}, we use an implementation through the pyCausalFS package.

\textbf{LDECC.} \ \ 
We use the default implementation provided by \citet{gupta2023local}.

\subsection{Performance Metrics}\label{sec:metrics}
In the acyclic case, we report two performance matrics, namely SHD of local DAG and PDAG. In the cyclic case, we report the SHD of local DCG. These metrics are explained in details below. For each setting, the metrics are calculated over $8$ random simulations.

\textbf{SHD of local DAG.} \ \  
For this metric, the ground-truth and estimated structures contain all incoming directed edges of $T$ and its children. We then compute the SHD between the ground-truth and estimated local structures. Such a metric is used to validate our method (see Theorem \ref{thm:correct_isa_ling}) that can estimate more edges and directions than the baselines. Note that the estimated output by Local A* and GSBN may contain undirected edges; therefore, we enumerate all possible combinations of directed edges from these undirected ones, and compute the final averaged SHD for these two methods.

\textbf{SHD of local PDAG.} \ \  
Since Local A* and GSBN return PDAG around the target $T$, we design this metric specifically for these baselines. In particular, the ground-truth and estimated local structures contain (1) undirected edges involving $T$ and (2) v-structures of which $T$ is not the collider. We then compute the SHD between the ground-truth and estimated local structures. Note that since our method returns DAG around the target $T$ that contains additional edges, we convert it into the same format of PDAG as well.

\textbf{SHD of local DCG.} \ \ 
This metric is similar to the SHD of local DAG explained above, except that the ground truth may contain cycles.

\section{SUPPLEMENTARY EXPERIMENTS RESULTS}\label{sec:supplementary_experiment_results}

\subsection{Running Time}\label{sec:running_time}
We report the running times for different sizes of MBs. Specifically, we follow the data generating procedure in \cref{sec:acyclic_experiments}. We generate random DAGs with expected degrees of $3$, $5$, and $7$, and select target with a relatively large MB. The running times of different methods, including those with and without oracle MBs, are shown in \cref{fig:running_time}. Note that, for Local A*, instances with MBs exceeding $19$ variables are omitted due to the long running time. It is observed that our method has a longer running time than that of LDECC and CMB. Furthermore, when size of MB increases, the running time of our method is shorter and increases much slowly compared to GSBN and Local A*.
\begin{figure}[!h]
\centering
\includegraphics[width=0.43\textwidth]{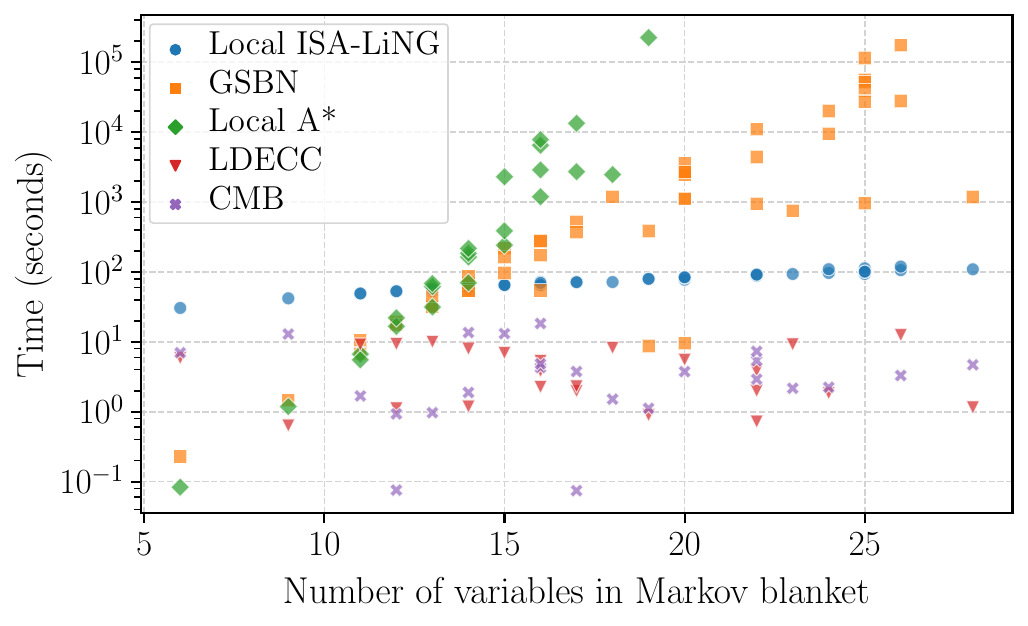}
\caption{Running time of different methods. Y-axis is in log scale.}
\label{fig:running_time}
\end{figure}

\subsection{Additional Figures}\label{sec:supplementary_experiment_figures}
This section provides additional figures for \cref{sec:experiments}, namely Figures \ref{fig:cyclic_graph_outputs}, \ref{fig:oracle_mb_dag_shd_cyclic}, \ref{fig:lasso_mb_shd_degree_3_acyclic}, \ref{fig:oracle_mb_shd_degree_3_acyclic}, \ref{fig:lasso_mb_shd_degree_5_acyclic}, \ref{fig:lasso_mb_different_sample_sizes}, \ref{fig:oracle_mb_edge_weights}, \ref{fig:lasso_mb_edge_weights}, and \ref{fig:sachs_exp}.

\begin{figure*}[!h]
\centering  
\subfloat[Ground-truth cyclic graph.]{
    \includegraphics[width=0.24\textwidth]{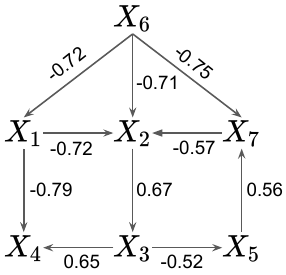}
}\\
\subfloat[Estimated graphs and edge weights by ICA-LiNG.]{
    \includegraphics[width=0.52\textwidth]{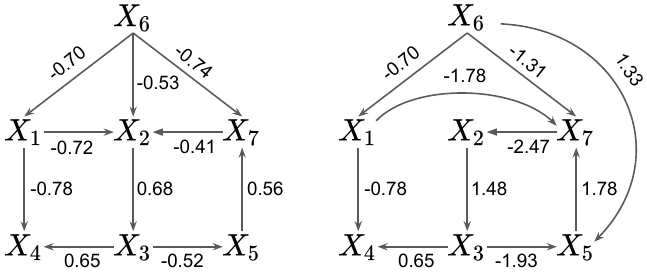}
}\\
\subfloat[Estimated graphs and edge weights by Local ISA-LiNG using $T=3$ as target.]{
    \includegraphics[width=0.52\textwidth]{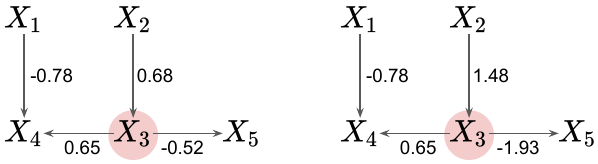}
}
\caption{Ground truth and estimated cyclic graphs and edge weights with $2000$ samples.}
\label{fig:cyclic_graph_outputs}
\end{figure*}

\begin{figure}[!h]
\centering
\includegraphics[width=0.35\textwidth]{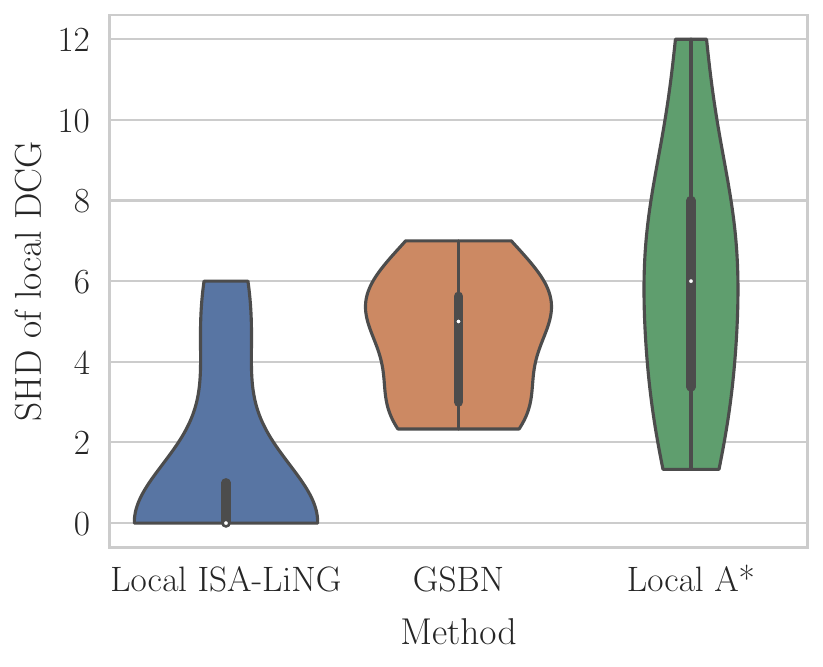}
\vspace{-0.4em}
\caption{SHD of local DCG under oracle MB.}
\label{fig:oracle_mb_dag_shd_cyclic}
\end{figure}

\begin{figure*}[!h]
\centering
\subfloat[SHD of local DAG.]{
    \includegraphics[width=0.4\textwidth]{figures/comparison_with_baselines/acyclic_case/degree_3/lasso_mb_comprehensive_dag_shd.pdf}
}\hspace{1em}
\subfloat[SHD of local PDAG.]{
    \includegraphics[width=0.4\textwidth]{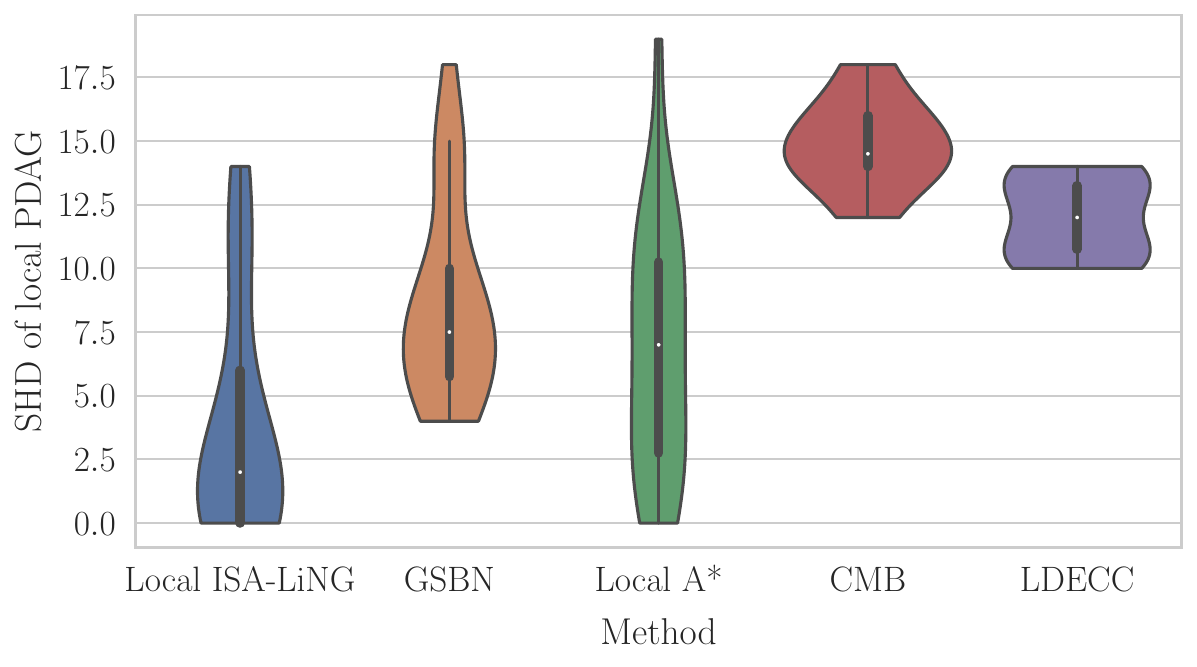}
    \label{fig:lasso_mb_pdag_shd_degree_3_acyclic}
}
\caption{Results of local causal discovery with $2000$ samples and degree of $3$ under estimated MB.}
\label{fig:lasso_mb_shd_degree_3_acyclic}
\end{figure*}

\begin{figure*}[!h]
\centering
\subfloat[SHD of local DAG.]{
    \includegraphics[width=0.35\textwidth]{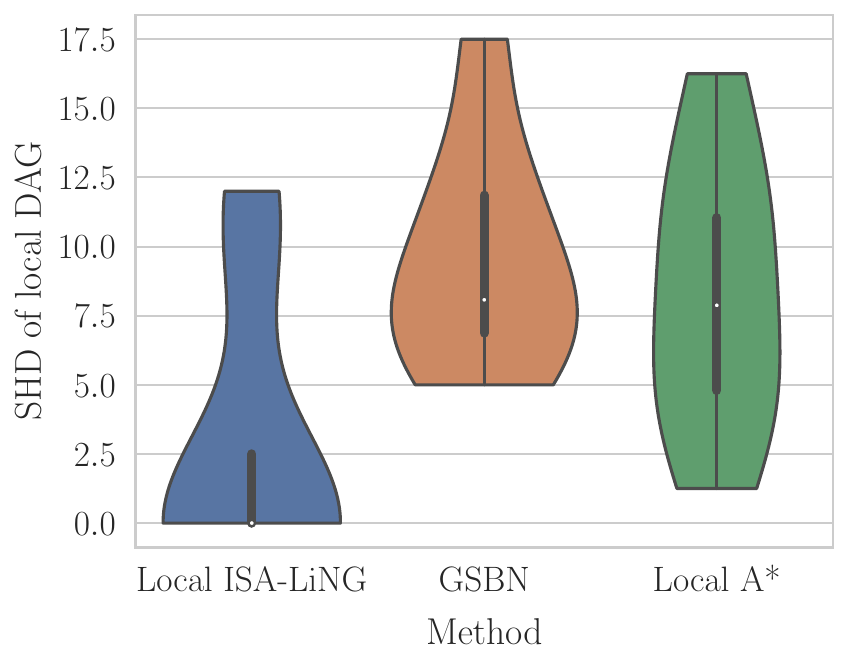}
    \label{fig:oracle_mb_dag_shd_degree_3_acyclic}
}\hspace{1em}
\subfloat[SHD of local PDAG.]{
    \includegraphics[width=0.35\textwidth]{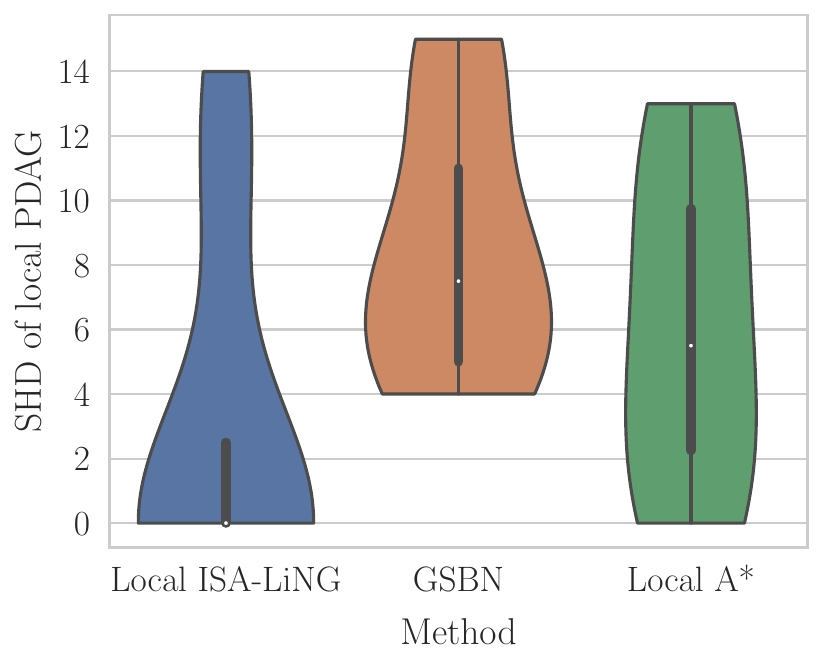}
    \label{fig:oracle_mb_pdag_shd_degree_3_acyclic}
}
\caption{Results of local causal discovery with $2000$ samples and degree of $3$ under oracle MB.}
\label{fig:oracle_mb_shd_degree_3_acyclic}
\end{figure*}

\begin{figure*}[!h]
\centering
\subfloat[SHD of local DAG.]{
    \includegraphics[width=0.4\textwidth]{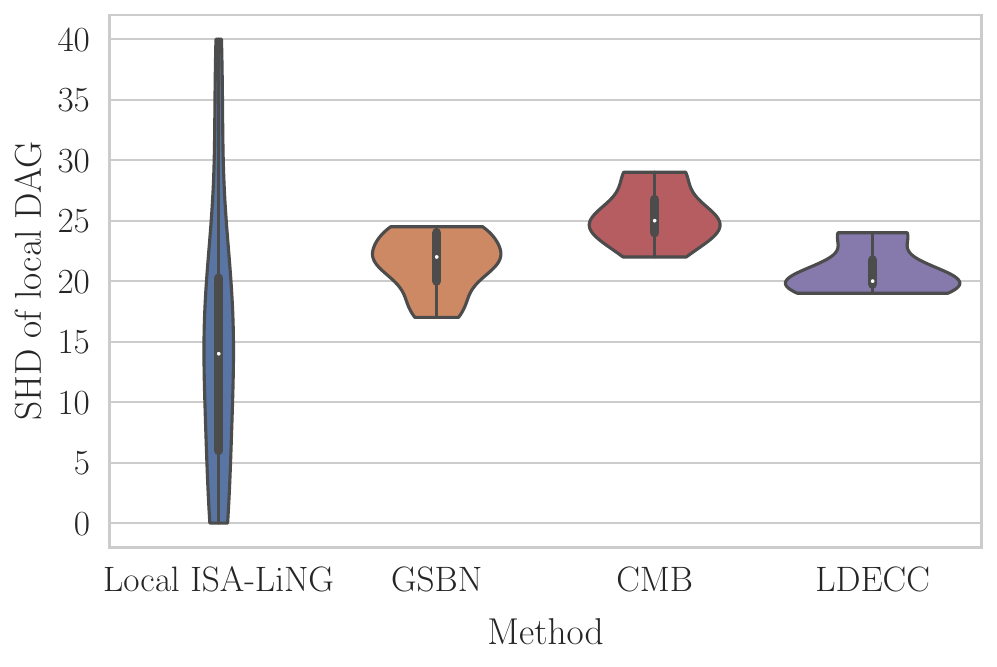}
    \label{fig:oracle_mb_dag_shd_degree_5_acyclic}
}\hspace{1em}
\subfloat[SHD of local PDAG.]{
    \includegraphics[width=0.4\textwidth]{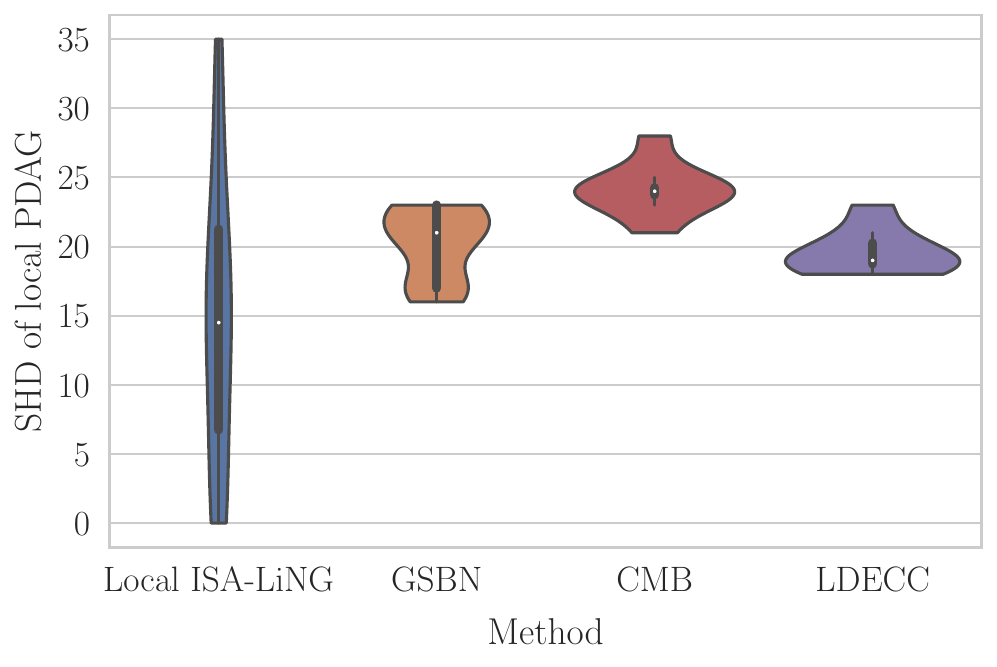}
    \label{fig:oracle_mb_pdag_shd_degree_5_acyclic}
}
\caption{Results of local causal discovery with $2000$ samples and degree of $5$ under estimated MB.}
\label{fig:lasso_mb_shd_degree_5_acyclic}
\end{figure*}

\begin{figure}[!h]
\centering
\includegraphics[width=0.4\textwidth]{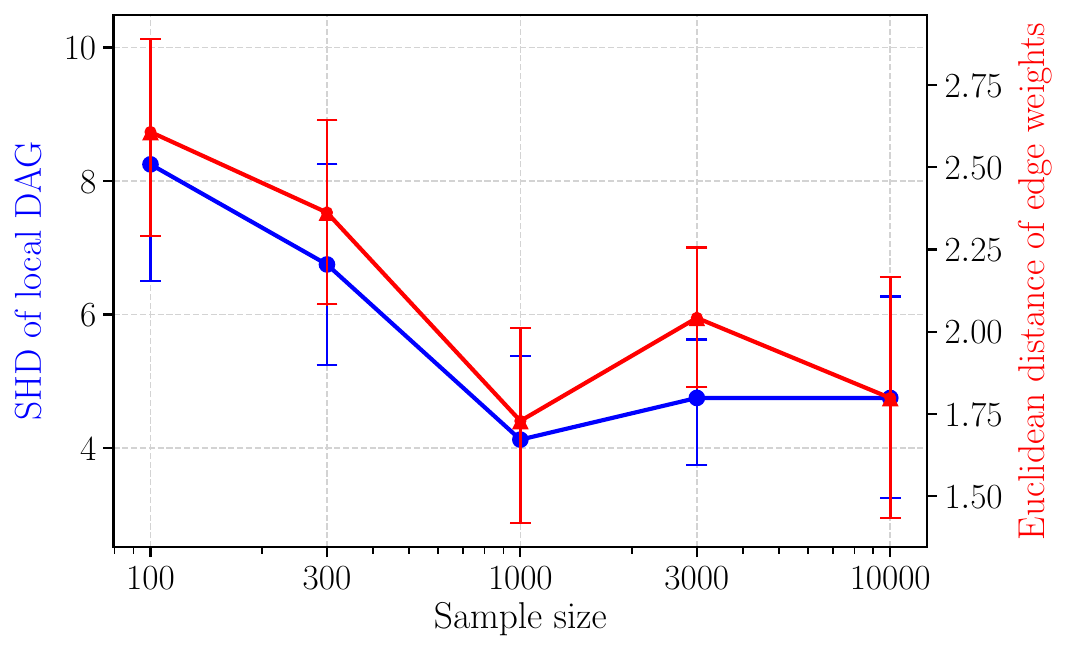}
\caption{Results of Local ISA-LiNG with MB estimated by Lasso. X-axis is visualized in log scale.}
\label{fig:lasso_mb_different_sample_sizes}
\end{figure}

\begin{figure*}[!t]
\centering  
\subfloat[$100$ samples.]{
    \includegraphics[width=0.3\textwidth]{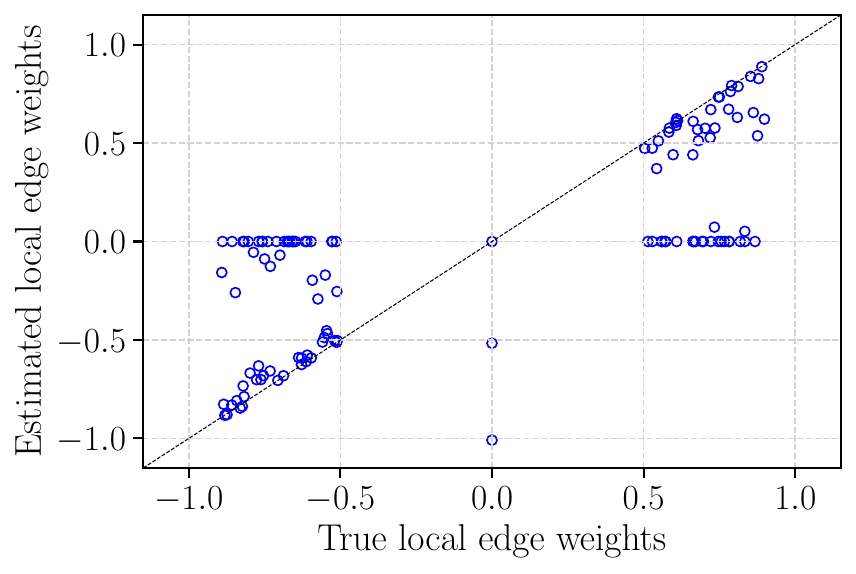}
    \label{fig:oracle_mb_edge_weights_100_samples}
}
\subfloat[$300$ samples.]{
    \includegraphics[width=0.3\textwidth]{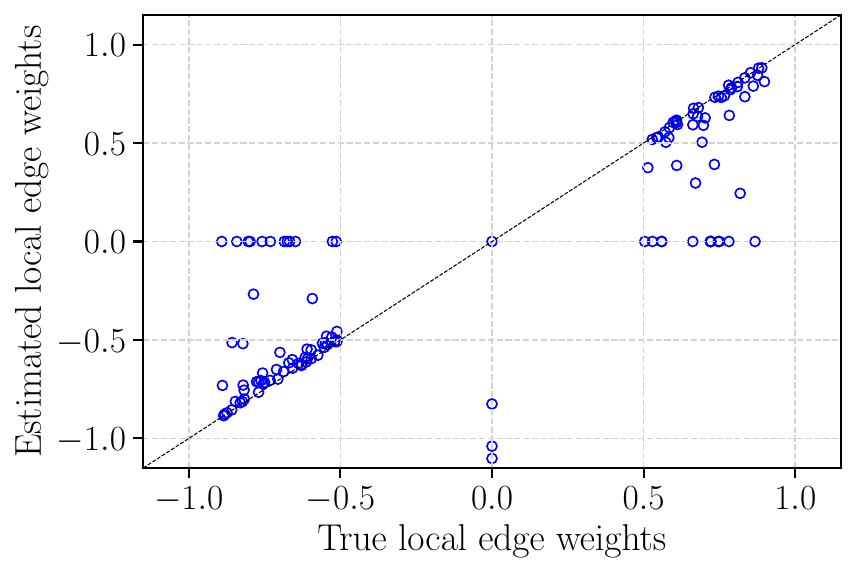}
    \label{fig:oracle_mb_edge_weights_300_samples}
}
\subfloat[$1000$ samples.]{
    \includegraphics[width=0.3\textwidth]{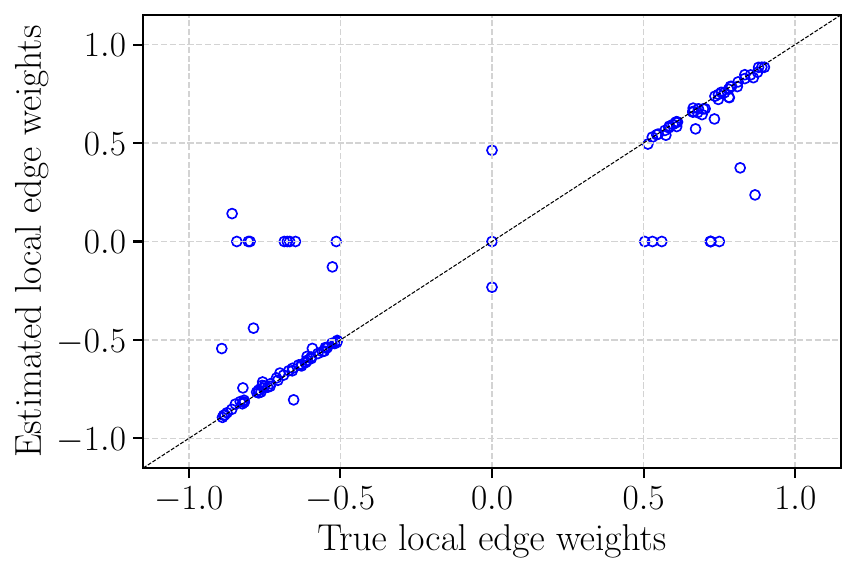}
    \label{fig:oracle_mb_edge_weights_1000_samples}
}
\caption{Edge weights estimated by Local ISA-LiNG under oracle MB.}
\label{fig:oracle_mb_edge_weights}
\end{figure*}

\begin{figure*}[!h]
\centering  
\subfloat[$100$ samples.]{
    \includegraphics[width=0.3\textwidth]{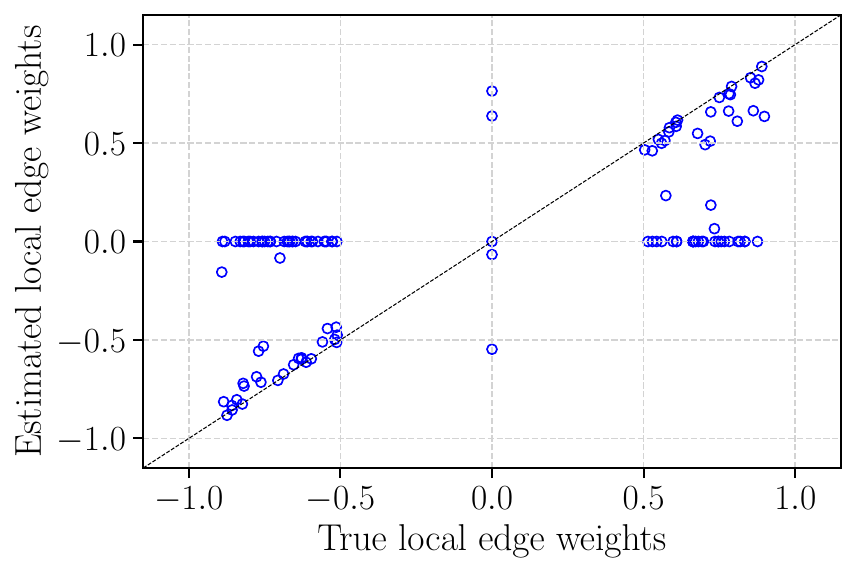}
    \label{fig:lasso_mb_edge_weights_100_samples}
}
\subfloat[$300$ samples.]{
    \includegraphics[width=0.3\textwidth]{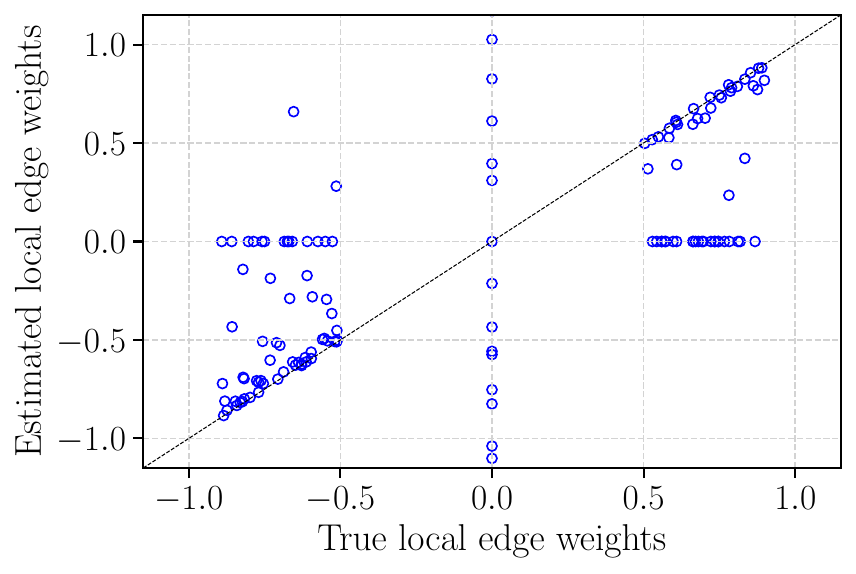}
    \label{fig:lasso_mb_edge_weights_300_samples}
}
\subfloat[$1000$ samples.]{
    \includegraphics[width=0.3\textwidth]{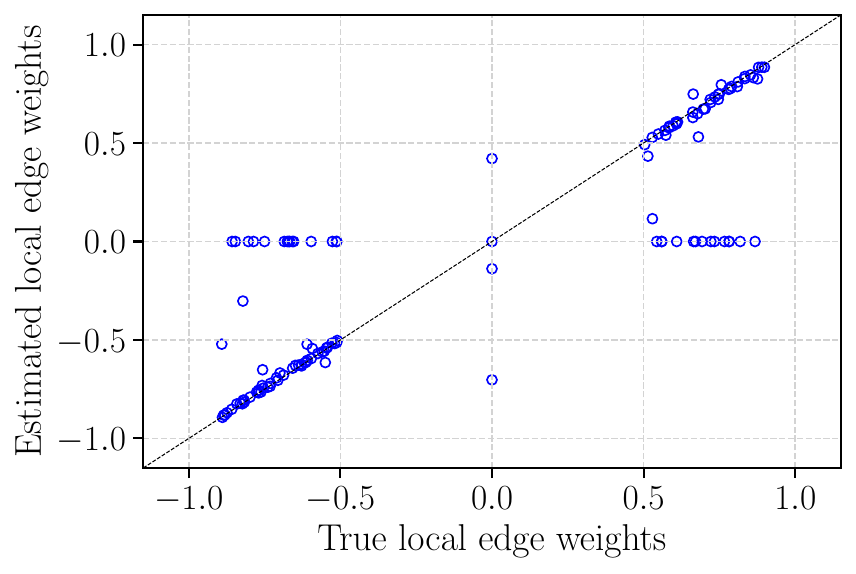}
    \label{fig:lasso_mb_edge_weights_1000_samples}
}
\caption{Edge weights estimated by Local ISA-LiNG under estimated MB.}
\label{fig:lasso_mb_edge_weights}
\end{figure*}

\begin{figure*}
    \centering
    \includegraphics[width=16cm]{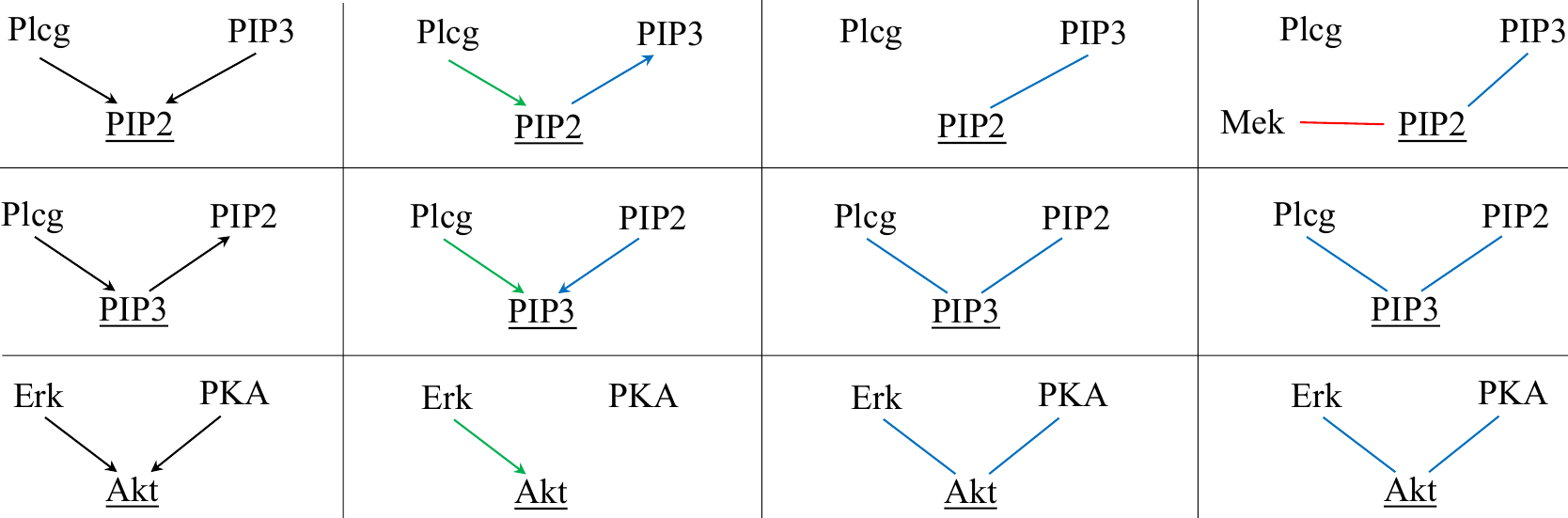}
    \caption{Result of local causal discovery on real-world dataset by \citet{sachs_data}. The first column showcases the ground-truth local causal graphs. From the second to the last column, each column corresponds to the local causal graphs recovered by (1) Local ISA-LiNG, (2) Local A$^{*}$, and (3) GSBN, respectively. We use underlined vertices to denote target variables. For the second column, green and blue arrows denote correct directed edges and wrong directed edges discovered by our method, respectively. For the third and fourth columns, blue lines denote correct undirected edges discovered by the baselines.}
    \label{fig:sachs_exp}
\end{figure*}


\begin{thebibliography}{49}
\providecommand{\natexlab}[1]{#1}
\providecommand{\url}[1]{\texttt{#1}}
\expandafter\ifx\csname urlstyle\endcsname\relax
  \providecommand{\doi}[1]{doi: #1}\else
  \providecommand{\doi}{doi: \begingroup \urlstyle{rm}\Url}\fi

\bibitem[Benito et~al.(2007)Benito, Zheng, Ng, and
  Hardin]{benito2007transcriptional}
J.~Benito, H.~Zheng, F.~S. Ng, and P.~E. Hardin.
\newblock Transcriptional feedback loop regulation, function and ontogeny in
  {Drosophila}.
\newblock In \emph{Cold Spring Harbor symposia on quantitative biology},
  volume~72, page 437. NIH Public Access, 2007.

\bibitem[Comon(1994)]{comon1994independent}
P.~Comon.
\newblock Independent component analysis, a new concept?
\newblock \emph{Signal processing}, 36\penalty0 (3):\penalty0 287--314, 1994.

\bibitem[Dai et~al.(2022)Dai, Spirtes, and Zhang]{dai2022independence}
H.~Dai, P.~Spirtes, and K.~Zhang.
\newblock Independence testing-based approach to causal discovery under
  measurement error and linear {Non-Gaussian} models.
\newblock \emph{Advances in Neural Information Processing Systems},
  35:\penalty0 27524--27536, 2022.

\bibitem[Darmois(1953)]{darmois1953analyse}
G.~Darmois.
\newblock Analyse g{\'e}n{\'e}rale des liaisons stochastiques: etude
  particuli{\`e}re de l'analyse factorielle lin{\'e}aire.
\newblock \emph{Revue de l'Institut international de statistique}, pages 2--8,
  1953.

\bibitem[Erd\"os and R\'enyi(1959)]{erdos1959random}
P.~Erd\"os and A.~R\'enyi.
\newblock On random graphs {I}.
\newblock \emph{Publicationes Mathematicae}, 6:\penalty0 290--297, 1959.

\bibitem[Friedman et~al.(2008)Friedman, Hastie, and
  Tibshirani]{friedman2008sparse}
J.~Friedman, T.~Hastie, and R.~Tibshirani.
\newblock Sparse inverse covariance estimation with the graphical {Lasso}.
\newblock \emph{Biostatistics}, 9:\penalty0 432--41, 2008.

\bibitem[Gao and Ji(2015)]{gao2015local}
T.~Gao and Q.~Ji.
\newblock Local causal discovery of direct causes and effects.
\newblock \emph{Advances in Neural Information Processing Systems}, 28, 2015.

\bibitem[Gao et~al.(2017)Gao, Fadnis, and Campbell]{gao2017local}
T.~Gao, K.~Fadnis, and M.~Campbell.
\newblock Local-to-global {Bayesian} network structure learning.
\newblock In \emph{International Conference on Machine Learning}, pages
  1193--1202. PMLR, 2017.

\bibitem[Ghassami et~al.(2020)Ghassami, Yang, Kiyavash, and
  Zhang]{ghassami2020characterizing}
A.~Ghassami, A.~Yang, N.~Kiyavash, and K.~Zhang.
\newblock Characterizing distribution equivalence and structure learning for
  cyclic and acyclic directed graphs.
\newblock In \emph{International Conference on Machine Learning}, pages
  3494--3504. PMLR, 2020.

\bibitem[Gretton et~al.(2007)Gretton, Fukumizu, Teo, Song, Sch\"{o}lkopf, and
  Smola]{gretton2007kernel}
A.~Gretton, K.~Fukumizu, C.~Teo, L.~Song, B.~Sch\"{o}lkopf, and A.~Smola.
\newblock A kernel statistical test of independence.
\newblock In \emph{Advances in Neural Information Processing Systems}, 2007.

\bibitem[Gupta et~al.(2023)Gupta, Childers, and Lipton]{gupta2023local}
S.~Gupta, D.~Childers, and Z.~C. Lipton.
\newblock Local causal discovery for estimating causal effects.
\newblock In \emph{Conference on Causal Learning and Reasoning}, pages
  408--447. PMLR, 2023.

\bibitem[Haavelmo(1943)]{haavelmo1943statistical}
T.~Haavelmo.
\newblock The statistical implications of a system of simultaneous equations.
\newblock \emph{Econometrica, Journal of the Econometric Society}, pages 1--12,
  1943.

\bibitem[Hoyer et~al.(2008)Hoyer, Shimizu, Kerminen, and
  Palviainen]{hoyer2008estimation}
P.~O. Hoyer, S.~Shimizu, A.~J. Kerminen, and M.~Palviainen.
\newblock Estimation of causal effects using linear {non-Gaussian} causal
  models with hidden variables.
\newblock \emph{International Journal of Approximate Reasoning}, 49\penalty0
  (2):\penalty0 362--378, 2008.

\bibitem[Hyttinen et~al.(2012)Hyttinen, Eberhardt, and
  Hoyer]{hyttinen2012learning}
A.~Hyttinen, F.~Eberhardt, and P.~O. Hoyer.
\newblock Learning linear cyclic causal models with latent variables.
\newblock \emph{The Journal of Machine Learning Research}, 13\penalty0
  (1):\penalty0 3387--3439, 2012.

\bibitem[Hyv{\"a}rinen and Hoyer(2000)]{hyvarinen2000emergence}
A.~Hyv{\"a}rinen and P.~Hoyer.
\newblock Emergence of phase-and shift-invariant features by decomposition of
  natural images into independent feature subspaces.
\newblock \emph{Neural computation}, 12\penalty0 (7):\penalty0 1705--1720,
  2000.

\bibitem[Hyv{\"a}rinen and Oja(2000)]{hyvarinen2000independent}
A.~Hyv{\"a}rinen and E.~Oja.
\newblock Independent component analysis: algorithms and applications.
\newblock \emph{Neural networks}, 13\penalty0 (4-5):\penalty0 411--430, 2000.

\bibitem[Janjic(2008)]{janjic2008proof}
M.~Janjic.
\newblock A proof of generalized {Laplace’s} expansion theorem.
\newblock \emph{Bull. Soc. Math. Banja Luka}, 15\penalty0 (2008):\penalty0
  5--7, 2008.

\bibitem[Lacerda et~al.(2008)Lacerda, Spirtes, Ramsey, and
  Hoyer]{lacerda2008discovering}
G.~Lacerda, P.~Spirtes, J.~Ramsey, and P.~O. Hoyer.
\newblock Discovering cyclic causal models by independent components analysis.
\newblock In \emph{Conference on Uncertainty in Artificial Intelligence}, 2008.

\bibitem[Levine and Davidson(2005)]{levine2005gene}
M.~Levine and E.~H. Davidson.
\newblock Gene regulatory networks for development.
\newblock \emph{Proceedings of the National Academy of Sciences}, 102\penalty0
  (14):\penalty0 4936--4942, 2005.

\bibitem[Ling et~al.(2020)Ling, Yu, Wang, Li, and Wu]{ling2020using}
Z.~Ling, K.~Yu, H.~Wang, L.~Li, and X.~Wu.
\newblock Using feature selection for local causal structure learning.
\newblock \emph{IEEE Transactions on Emerging Topics in Computational
  Intelligence}, 5\penalty0 (4):\penalty0 530--540, 2020.

\bibitem[Loh and B{{\"u}}hlmann(2014)]{loh2014high}
P.-L. Loh and P.~B{{\"u}}hlmann.
\newblock High-dimensional learning of linear causal networks via inverse
  covariance estimation.
\newblock \emph{Journal of Machine Learning Research}, 15\penalty0
  (88):\penalty0 3065--3105, 2014.

\bibitem[Ma et~al.(2023)Ma, Wang, Bieganek, Tourani, and Aliferis]{ma2023local}
S.~Ma, J.~Wang, C.~Bieganek, R.~Tourani, and C.~Aliferis.
\newblock Local causal pathway discovery for single-cell rna sequencing count
  data: a benchmark study.
\newblock \emph{Journal of Translational Genetics and Genomics}, 7\penalty0
  (1):\penalty0 50--65, 2023.

\bibitem[Margaritis and Thrun(1999)]{margaritis1999bayesian}
D.~Margaritis and S.~Thrun.
\newblock Bayesian network induction via local neighborhoods.
\newblock \emph{Advances in neural information processing systems}, 12, 1999.

\bibitem[Mason(1953)]{mason1953feedback}
S.~J. Mason.
\newblock Feedback theory-some properties of signal flow graphs.
\newblock \emph{Proceedings of the IRE}, 41\penalty0 (9):\penalty0 1144--1156,
  1953.

\bibitem[Meinshausen and B{\"u}hlmann(2006)]{meinshausen2006high}
N.~Meinshausen and P.~B{\"u}hlmann.
\newblock High-dimensional graphs and variable selection with the {Lasso}.
\newblock \emph{The Annals of Statistics}, 34:\penalty0 1436--1462, 2006.

\bibitem[Mooij and Heskes(2013)]{mooij2013cyclic}
J.~Mooij and T.~Heskes.
\newblock Cyclic causal discovery from continuous equilibrium data.
\newblock \emph{arXiv preprint arXiv:1309.6849}, 2013.

\bibitem[Ng et~al.(2021)Ng, Zheng, Zhang, and Zhang]{ng2021reliable}
I.~Ng, Y.~Zheng, J.~Zhang, and K.~Zhang.
\newblock Reliable causal discovery with improved exact search and weaker
  assumptions.
\newblock In \emph{Advances in Neural Information Processing Systems
  (NeurIPS)}, 2021.

\bibitem[Niinimaki and Parviainen(2012)]{niinimaki2012local}
T.~Niinimaki and P.~Parviainen.
\newblock Local structure discovery in {Bayesian} networks.
\newblock \emph{arXiv preprint arXiv:1210.4888}, 2012.

\bibitem[Ravikumar et~al.(2011)Ravikumar, Wainwright, Raskutti, and
  Yu]{ravikumar2011high}
P.~Ravikumar, M.~J. Wainwright, G.~Raskutti, and B.~Yu.
\newblock High-dimensional covariance estimation by minimizing
  $\ell_1$-penalized log-determinant divergence.
\newblock \emph{Electronic Journal of Statistics}, 5, 2011.

\bibitem[Richardson(1996)]{richardson1996discovering}
T.~S. Richardson.
\newblock \emph{Discovering cyclic causal structure}.
\newblock Carnegie Mellon [Department of Philosophy], 1996.

\bibitem[Rothenh{\"a}usler et~al.(2015)Rothenh{\"a}usler, Heinze, Peters, and
  Meinshausen]{rothenhausler2015backshift}
D.~Rothenh{\"a}usler, C.~Heinze, J.~Peters, and N.~Meinshausen.
\newblock Backshift: Learning causal cyclic graphs from unknown shift
  interventions.
\newblock \emph{Advances in Neural Information Processing Systems}, 28, 2015.

\bibitem[Sachs et~al.(2005)Sachs, Perez, Pe'er, Lauffenburger, and
  Nolan]{sachs_data}
K.~Sachs, O.~Perez, D.~Pe'er, D.~A. Lauffenburger, and G.~P. Nolan.
\newblock Causal protein-signaling networks derived from multiparameter
  single-cell data.
\newblock \emph{Science}, 308\penalty0 (5721):\penalty0 523--529, 2005.

\bibitem[Sanchez-Romero et~al.(2019)Sanchez-Romero, Ramsey, Zhang, Glymour,
  Huang, and Glymour]{sanchez2019estimating}
R.~Sanchez-Romero, J.~D. Ramsey, K.~Zhang, M.~R. Glymour, B.~Huang, and
  C.~Glymour.
\newblock Estimating feedforward and feedback effective connections from {fMRI}
  time series: Assessments of statistical methods.
\newblock \emph{Network Neuroscience}, 3\penalty0 (2):\penalty0 274--306, 2019.

\bibitem[Schwarz(1978)]{schwarz1978estimating}
G.~Schwarz.
\newblock Estimating the dimension of a model.
\newblock \emph{The Annals of Statistics}, 6\penalty0 (2):\penalty0 461--464,
  1978.

\bibitem[Shimizu et~al.(2006)Shimizu, Hoyer, Hyv{\"a}rinen, and
  Kerminen]{shimizu2006lingam}
S.~Shimizu, P.~O. Hoyer, A.~Hyv{\"a}rinen, and A.~Kerminen.
\newblock A linear {non-Gaussian} acyclic model for causal discovery.
\newblock \emph{Journal of Machine Learning Research}, 7\penalty0
  (Oct):\penalty0 2003--2030, 2006.

\bibitem[Shimizu et~al.(2011)Shimizu, Inazumi, Sogawa, Hyv{\"a}rinen, Kawahara,
  Washio, Hoyer, and Bollen]{shimizu2011directlingam}
S.~Shimizu, T.~Inazumi, Y.~Sogawa, A.~Hyv{\"a}rinen, Y.~Kawahara, T.~Washio,
  P.~O. Hoyer, and K.~Bollen.
\newblock {DirectLiNGAM}: A direct method for learning a linear {non-Gaussian}
  structural equation model.
\newblock \emph{Journal of Machine Learning Research}, 12\penalty0
  (Apr):\penalty0 1225--1248, 2011.

\bibitem[Skitovitch(1953)]{skitovitch1953property}
V.~P. Skitovitch.
\newblock On a property of the normal distribution.
\newblock \emph{DAN SSSR}, 89:\penalty0 217--219, 1953.

\bibitem[Spirtes(1994)]{spirtes1994conditional}
P.~Spirtes.
\newblock Conditional independence in directed cyclic graphical models
  representing feedback or mixtures.
\newblock Technical report, Philosophy, Methodology and Logic Technical Report
  59, CMU, 1994.

\bibitem[Spirtes(1995)]{spirtes1995directed}
P.~Spirtes.
\newblock Directed cyclic graphical representations of feedback models.
\newblock In \emph{Conference on Uncertainty in Artificial Intelligence}, 1995.

\bibitem[Theis(2006)]{theis2006general}
F.~Theis.
\newblock Towards a general independent subspace analysis.
\newblock In \emph{Advances in Neural Information Processing Systems}, 2006.

\bibitem[Tibshirani(1996)]{robert1996lasso}
R.~Tibshirani.
\newblock Regression shrinkage and selection via the {Lasso}.
\newblock \emph{Journal of the Royal Statistical Society. Series B
  (Methodological)}, 58\penalty0 (1):\penalty0 267--288, 1996.

\bibitem[Tsamardinos et~al.(2003)Tsamardinos, Aliferis, and
  Statnikov]{tsamardinos2003algorithms}
I.~Tsamardinos, C.~Aliferis, and A.~Statnikov.
\newblock Algorithms for large scale {Markov} blanket discovery.
\newblock pages 376--381, 2003.

\bibitem[Tsamardinos et~al.(2006)Tsamardinos, Brown, and
  Aliferis]{tsamardinos2006mmhc}
I.~Tsamardinos, L.~Brown, and C.~Aliferis.
\newblock The max-min hill-climbing {Bayesian} network structure learning
  algorithm.
\newblock \emph{Machine Learning}, 65:\penalty0 31--78, 10 2006.

\bibitem[Wang et~al.(2014)Wang, Zhou, Zhao, and Geng]{wang2014discovering}
C.~Wang, Y.~Zhou, Q.~Zhao, and Z.~Geng.
\newblock Discovering and orienting the edges connected to a target variable in
  a {DAG} via a sequential local learning approach.
\newblock \emph{Computational Statistics \& Data Analysis}, 77:\penalty0
  252--266, 2014.

\bibitem[Yin et~al.(2008)Yin, Zhou, Wang, He, Zheng, and Geng]{yin2008partial}
J.~Yin, Y.~Zhou, C.~Wang, P.~He, C.~Zheng, and Z.~Geng.
\newblock Partial orientation and local structural learning of causal networks
  for prediction.
\newblock In \emph{Causation and Prediction Challenge}, pages 93--105. PMLR,
  2008.

\bibitem[Yu et~al.(2021)Yu, Ling, Liu, Li, Wang, and Li]{yu2021feature}
K.~Yu, Z.~Ling, L.~Liu, P.~Li, H.~Wang, and J.~Li.
\newblock Feature selection for efficient local-to-global {Bayesian} network
  structure learning.
\newblock \emph{ACM Transactions on Knowledge Discovery from Data}, 2021.

\bibitem[Yuan and Malone(2013)]{yuan2013learning}
C.~Yuan and B.~Malone.
\newblock Learning optimal {Bayesian} networks: A shortest path perspective.
\newblock \emph{Journal of Artificial Intelligence Research}, 48\penalty0
  (1):\penalty0 23--65, 2013.

\bibitem[Zhang and Chan(2006)]{zhang2006ica}
K.~Zhang and L.-W. Chan.
\newblock {ICA} with sparse connections.
\newblock In \emph{International Conference on Intelligent Data Engineering and
  Automated Learning}, pages 530--537. Springer, 2006.

\bibitem[Zhou et~al.(2010)Zhou, Wang, Yin, and Geng]{zhou2010discover}
Y.~Zhou, C.~Wang, J.~Yin, and Z.~Geng.
\newblock Discover local causal network around a target to a given depth.
\newblock In \emph{Causality: Objectives and Assessment}, pages 191--202. PMLR,
  2010.

\end{thebibliography}
\end{document}